\documentclass[conference,compsoc]{IEEEtran}

\usepackage[nocompress]{cite}
\usepackage{ifpdf}
\usepackage{graphicx}
\usepackage{times,epsfig}
\usepackage{epstopdf}
\usepackage{amsfonts,amssymb,amsmath,amsthm, bm}
\usepackage{algorithm}
\usepackage{algorithmic}
\usepackage{array}
\usepackage[caption=false,font=footnotesize,labelfont=sf,textfont=sf]{subfig}
\usepackage{booktabs}
\usepackage{multirow}
\usepackage{makecell}
 \usepackage{xcolor}

\newtheorem{definition}{Definition}

\newtheorem{theorem}{Theorem}
\newtheorem{remark}{Remark}
\newtheorem{problem}{Problem}
\newtheorem{challenge}{Challenge}

\begin{document}
%
\title{Efficiently Achieving Secure Model Training and Secure Aggregation to Ensure Bidirectional Privacy-Preservation in Federated Learning}


\author{\IEEEauthorblockN{Xue Yang\IEEEauthorrefmark{1},
Depan Peng\IEEEauthorrefmark{1},
Yan Feng\IEEEauthorrefmark{2}, 
Xiaohu Tang\IEEEauthorrefmark{1} (\emph{Fellow, IEEE}), 
Weijun Fang\IEEEauthorrefmark{3}\IEEEauthorrefmark{5} and
Jun Shao\IEEEauthorrefmark{4}}
\IEEEauthorblockA{\IEEEauthorrefmark{1}School of Information Science and Technology, Southwest Jiaotong University, Chengdu, China\\
Email: xueyang@swjtu.edu.cn, pengdepan@my.swjtu.edu.cn and xhutang@swjtu.edu.cn}
\IEEEauthorblockA{\IEEEauthorrefmark{2}Meituan Group, Shanghai, China\\
Email: fengyan14@meituan.com}
\IEEEauthorblockA{\IEEEauthorrefmark{3}Shandong University, Shandong, China\\
Email: fwj@sdu.edu.cn}
\IEEEauthorblockA{\IEEEauthorrefmark{4}Zhejiang Gongshang University, Hangzhou, China\\
Email: chn.junshao@gmail.com}
\IEEEauthorblockA{\IEEEauthorrefmark{5}The corresponding author}}


%

\maketitle

\begin{abstract}
Bidirectional privacy-preservation federated learning is crucial as both local gradients and the global model may leak privacy. However, only a few works attempt to achieve it, and they often face challenges such as excessive communication and computational overheads, or significant degradation of model accuracy, which hinders their practical applications. In this paper, we design an efficient and high-accuracy bidirectional privacy-preserving scheme for federated learning to complete secure model training and secure aggregation. To efficiently achieve bidirectional privacy, we design an efficient and accuracy-lossless model perturbation method on the server side (called $\mathbf{MP\_Server}$) that can be combined with local differential privacy (LDP) to prevent clients from accessing the model, while ensuring that the local gradients obtained on the server side satisfy LDP. Furthermore, to ensure model accuracy, we customize a distributed differential privacy mechanism on the client side (called $\mathbf{DDP\_Client}$). When combined with $\mathbf{MP\_Server}$, it ensures LDP of the local gradients, while ensuring that the aggregated result matches the accuracy of central differential privacy (CDP). Extensive experiments demonstrate that our scheme significantly outperforms state-of-the-art bidirectional privacy-preservation baselines (SOTAs) in terms of computational cost, model accuracy, and defense ability against privacy attacks. Particularly, given target accuracy, the training time of SOTAs is approximately $200$ times, or even over $1000$ times, longer than that of our scheme. When the privacy budget is set relatively small, our scheme incurs less than $6\%$ accuracy loss compared to the privacy-ignoring method, while SOTAs suffer up to $20\%$ accuracy loss. Experimental results also show that the defense capability of our scheme outperforms than SOTAs.
\end{abstract}


%
\IEEEpeerreviewmaketitle

\section{Introduction}

Federated learning (FL) \cite{McMahanMRHA17} has emerged as a novel paradigm in machine learning that enables multiple decentralized clients to collaboratively train a global model under the coordination of a central server without exchanging their local data. Unlike traditional centralized learning methods, where data is collected and processed on a central server, each client trains its model on local data and sends only the computed local gradients to the central server. The server aggregates these local gradients to update the global model, which is then redistributed to the clients for the next training iteration. This approach is particularly well-suited for scenarios where privacy or data ownership concerns prevent the centralized aggregation of data, such as in healthcare and mobile applications. By ensuring that sensitive data remains on clients, FL reduces the risk of data breaches. 

Despite its advantages, FL remains vulnerable to privacy attacks \cite{ZhuLH19, GeipingBD020, ShokriSSS17}, such as model inversion and data leakage, when sharing local gradients and the global model. On one hand, local gradients and the global model are derived from clients' training data, which may contain sensitive information. Attackers can exploit this to infer private details (e.g., membership or class representation) or even reconstruct the original training data \cite{ZhuLH19, NasrSH19}. On the other hand, in practice, the central server is often managed by an organization or enterprise that compensates clients for assisting in training a high-quality global model for predictive services \cite{MelisSCS19, LiSTS20}. In such cases, the global model is considered a valuable asset and must be protected from unauthorized access, including from the clients themselves. \emph{Therefore, in practice, it is essential to protect both the local gradients from each client and the global model on the server, ensuring bidirectional privacy preservation}.

However, various existing privacy-preserving federated learning (PPFL) schemes either protect clients' local gradients from server access \cite{BonawitzIKMMPRS17,  ZhouYHDX23, AbadiCGMMT016} or prevent clients from accessing the global model \cite{0003FFSTXL22, MandalG19}, yet they do not offer bidirectional privacy-preservation. As far as we know, only a very limited number of PPFL schemes \cite{TruexBASLZZ19, XuBZAL19, TangSLZXQ23} consider combining secure multi-party computation techniques \cite{Shamir79, Paillier99} with differential privacy \cite{DworkKMMN06} or perturbation methods \cite{0003FFSTXL22} to protect bidirectional privacy. Nevertheless, they only focus on protecting the global model from the last iteration (i.e., the trained global model) and do not consider safeguarding it during intermediate iterations. This means that during these intermediate stages, the server directly sends the global model in plaintext, which still does not fully address the privacy issues surrounding the global model. Furthermore, the computational cost and communication overhead of these schemes are too high to be applied to FL, particularly in cross-device FL, where clients consist of numerous mobile or IoT devices that often have unreliable connections (e.g., Wi-Fi) and limited computing power.

As demonstrated in \cite{LiSTS20}, designing an efficient and bidirectional privacy-preserving federated learning scheme remains an unresolved challenge. In this paper, we address this issue by presenting Theorem \ref{defactorization} and subsequently designing an efficient and accuracy-lossless \emph{model perturbation method} on the server side ($\mathbf{MP\_Server}$), along with a customized \emph{distributed differential privacy} mechanism on the client side ($\mathbf{DDP\_Client}$). The synergy between $\mathbf{MP\_Server}$ and $\mathbf{DDP\_Client}$ enables secure model training and secure aggregation while maintaining model accuracy. Furthermore, our scheme operates in the real number domain, making it highly efficient and easy to integrate. We confidently assert that our approach is currently the most efficient solution in the field of bidirectional PPFL research. Our main contributions are threefold:
\begin{itemize}
    \item For bidirectional privacy, we design $\mathbf{MP\_Server}$ that can be combined with LDP to prevent clients from accessing the model, while ensuring that the local gradients obtained on the server side satisfy LDP. Inspired by the high efficiency of DP, $\mathbf{MP\_Server}$ allows the server to select carefully designed random variables as perturbed noises added to the global model, and thus it is very efficient and easy to integrate in practice. Unlike DP, $\mathbf{MP\_Server}$ allows the server to eliminate noises of local gradients to maintain accuracy. 

    
    \item To reduce the accuracy loss caused by LDP, we elegantly customize  $\mathbf{DDP\_Client}$, which includes two different types of noises: pairwise-correlated noises that can eventually be canceled after aggregation and independent (non-canceling) noises that would be needed to protect the true aggregated result. Given the privacy budget, $\mathbf{DDP\_Client}$ makes these pairwise correlated noises large enough so that the variance of independent (non-cancelling) noises remains minimal, allowing the aggregated result to match the accuracy of CDP.

    \item To compare the efficiency, model accuracy, and privacy of our scheme with SOTAs (i.e., privacy-ignoring FedAvg \cite{McMahanMRHA17} and three bidirectional privacy-preserving baselines, TP-SMC \cite{TruexBASLZZ19}, PILE \cite{TangSLZXQ23} and HybridAlpha \cite{XuBZAL19}), we conduct extensive experiments on large-scale image datasets. Empirical results of computational efficiency demonstrate that the computational efficiency of our scheme is almost comparable to that of FedAvg \cite{McMahanMRHA17}, and significantly outperforms TP-SMC \cite{TruexBASLZZ19}, PILE \cite{TangSLZXQ23} and HybridAlpha \cite{XuBZAL19}. The comparison of model accuracy shows that our scheme can achieve the same utility as the CDP and is much better than SOTAs. Particularly, when the privacy budget $\epsilon$ is set relatively small (e.g., $\epsilon=1$), the accuracy loss of our scheme is less than $6\%$ compared to the FedAvg, while the others suffer up to $20\%$ accuracy loss. Furthermore, experimental results on defending against privacy attacks show that the defense capability of our proposed scheme also outperforms than SOTAs.
\end{itemize}

The remainder of this paper is organized as follows. We will discuss related works in Section \ref{sec:related}, followed by the relevant background knowledge and two key problems in Section \ref{sec:prelimi}. We will give the details about our scheme in Section \ref{sec:MP_DP}, followed by its privacy analysis and performance evaluation in Sections \ref{sec:privacy_analysis} and \ref{sec:performance}, respectively. Finally, we conclude our work in Section \ref{sec:conclusion}. 

\section{Related Work}\label{sec:related}
Although FL has largely addressed the security issues of traditional centralized learning, attackers can still obtain some private information from the transmitted gradients or global model through privacy attacks such as reconstruction or inference attacks \cite{ZhuLH19, GeipingBD020, ShokriSSS17}. Thus, many PPFL schemes have been presented, which can be mainly divided into three categories: 1) the privacy of local gradients (i.e., secure aggregation), 2) the privacy of global model, and 3) the privacy of both local gradients and global model.

\underline{\textbf{\emph{1) Privacy of local gradients.}}} As proved in \cite{ZhuLH19}, the adversary (including the semi-trusted server) can reconstruct the training data from the local gradients. Thus, many secure aggregation schemes \cite{BonawitzIKMMPRS17, PhongAHWM18, SoNYL0AGA22, AbadiCGMMT016, ZhangLX00020, abs-2006-07218, LiYZL24, ZhuLLYXL22, WeiL22} are presented to guarantee the privacy of local gradients through either secure multi-party computation (SMC) technique (e.g., secret sharing (SS) \cite{Shamir79} and homomorphic encryption (HE) \cite{Paillier99}) or differential privacy (DP) \cite{DworkKMMN06}. Specifically, the schemes \cite{BonawitzIKMMPRS17, SoNYL0AGA22, LiYZL24} mask the local gradients with pairwise-correlated noises generated by SS and PRG techniques. Alternatively, since the server only performs weighted average calculations, additive HE-based secure aggregation schemes \cite{PhongAHWM18, ZhangLX00020} are developed, where each client encrypts its local gradients before uploading. However, these SMC-based secure aggregation schemes suffer from high computational and communication costs, making their practical applications far from FL. Consequently, DP, as a very efficient privacy-preserving technique, has been widely employed to ensure privacy \cite{AbadiCGMMT016, ZhuLLYXL22, WeiL22}. However, almost all secure aggregation schemes do not consider preventing clients from obtaining the global model.

\underline{\textbf{\emph{2) Privacy of global model.}}} While much attention has been paid to protecting local gradients, safeguarding the global model from unauthorized access by clients is equally important, especially when the global model is considered a valuable asset of the central server. What's more, the adversary (including semi-trusted clients) may infer whether a certain data record is a part of (other clients') training set through membership inference attacks with the global model parameters \cite{ShokriSSS17}. Thus, this requires that clients cannot train on the plaintext global model. As far as we know, many privacy-preserving machine learning schemes \cite{JiaAZJC22, 0002SKG19} have been developed to protect model parameters during traditional centralized training. These schemes mainly adopt SMC to compute the matrix multiplication and non-linear activation functions like ReLU under the ciphertext domain. Similarly, the corresponding computational costs and communication overheads (especially for the number of communication rounds) are too high to be applied in FL. Furthermore, no research applies these schemes in FL to preserve the privacy of the global model. To overcome these drawbacks, a very efficient and accuracy-lossless model perturbation method \cite{0003FFSTXL22} has been introduced recently, which perturbs the global model with randomly selected real numbers to let clients train on the perturbed model. Although this method effectively protects the global model's privacy, to ensure accuracy, the server can eliminate the random numbers to reconstruct each client's local gradient, thereby compromising the privacy of the local gradients.

\underline{\textbf{\emph{3) Privacy of both local gradients and global model.}}} To the best of our knowledge, only a few works have considered the privacy of both the global model and the local gradients \cite{TruexBASLZZ19, XuBZAL19, TangSLZXQ23}. For example, both \cite{TruexBASLZZ19} and \cite{XuBZAL19} combine SMC (e.g., either the threshold HE \cite{DamgardJ01} or functional encryption \cite{BonehSW11}) with DP to achieve the privacy of local gradients and global model. Specifically, each client first adds the DP noises to the local gradients and then encrypts the DP-perturbed local gradients through the threshold HE and functional encryption in \cite{TruexBASLZZ19} and \cite{XuBZAL19}, respectively. After that, the server can decrypt them to obtain the updated global model. Although only the server can perform the decryption operation, in both schemes, the server only protects the global model of the last iteration (i.e., the well-trained global model) rather than the global model during the middle iterations. In other words, the server directly sends the plaintext global model to each client for training, which violates the server-side privacy requirement. Subsequently, the scheme \cite{TangSLZXQ23} combines the model perturbation method \cite{0003FFSTXL22} with the threshold HE \cite{DamgardJ01} to achieve the complete privacy of both local gradients and global model during the entire iterations. However, this scheme is customized for the cross-silo FL and is not applicable to FL scenarios with large-scale clients.

\emph{As far as we know, there is no efficient solution that simultaneously achieves secure training and aggregation.}

\section{Background and Problem Statement}\label{sec:prelimi}
In this section, we first provide a brief introduction of the basic mechanisms adopted in our method, and then highlight the key problems of our scheme.

\subsection{Differential Privacy}
Differential privacy (DP) is widely accepted as the golden standard for data privacy \cite{DworkR14}. The goal of DP is to constrain the impact of replacing a single data record from the training set so that it is almost impossible to infer private information about a single record from querying the resulting models. In principle, DP quantifies the privacy loss by parameters $\epsilon$ and $\delta$, and the smaller values of the parameters correspond to less privacy loss.

\begin{definition} {(Differential Privacy.)} A randomized algorithm $\mathcal{M}$ is $(\epsilon, \delta)$-differential private if and only if for any subset $S \subset $ Range $\mathcal{R}$, and for all neighboring datasets $D$ and $D'$ differing by at most one record, we have:
  \begin{align}
        \Pr[\mathcal{M}(D) \in S] \leq e^{\epsilon} \Pr[\mathcal{M}(D') \in S] + \delta
    \end{align}\label{def:dp}
   where $\epsilon$ is the {privacy budget} and $\delta$ is the {failure probability}.
\end{definition}
From the above definition, we can see that DP protects the privacy of clients by masking the contribution of a single client for the resulting model, so that adversaries cannot infer much information from querying the trained model.

One of the most commonly used differential privacy mechanism for approximating a deterministic real-valued function ($f\colon \mathcal{D} \rightarrow \mathbb{R}$) is Gaussian mechanism, i.e., adding Gaussian noises calibrated to $f$'s \emph{sensitivity} $S_f$, which is defined as follow: for any pair of neighboring datasets $D$ and $D'$, the sensitivity of function $f$, denoted by $S_{f}$, is
\begin{equation}\label{eq:sensitive}
  S_{f}=\max\limits_{D, D'}\|f(D)-f(D')\|_{1}.
\end{equation}
Furthermore, Gaussian noise mechanism $\mathcal{M}$ is defined as:
\begin{equation} \label{DP}
    \mathcal{M}(d) \triangleq f(d) + \eta
\end{equation}
where $\eta \sim \mathcal{N} (0, S^2_f \sigma^2)$ is the random variable that obeys the Gaussian distribution with mean 0 and standard deviation $S_f \sigma$. A single application of the Gaussian mechanism to function $f$ of sensitivity
$S_f$ satisfies $(\epsilon, \delta)$-differential privacy if $\delta \geq \frac{5}{4}exp(-(\sigma\epsilon)^2/2)$ and $\epsilon < 1$~\cite{DworkR14}.

\begin{remark}
The reason why DP cannot provide bidirectional privacy preservation is that it primarily focuses on protecting the local gradients of clients in FL, while failing to prevent clients from accessing the server's global model. Given the privacy budget, clients are able to obtain the same global model as the server for prediction purposes. Furthermore, DP typically involves a trade-off, where enhanced privacy comes at the cost of reduced model accuracy.
\end{remark}

\subsection{Model Perturbation Method}\label{subsec:MP}
Recently, \cite{0003FFSTXL22} proposed a model perturbation method that allows clients to execute training over the private model. Specifically, consider a multi-layer Perceptron (MLP) with $L$ layers, where the non-linear activation function and loss function are $\mathrm{ReLU}(x)=\max(0, x)$ and the mean square error (MSE), respectively. The overall framework, which includes $N$ clients indexed by $\{1,2, \ldots, N\}$ and a central server, is summarized as follows.

\subsubsection{Global Model Perturbation $(\widehat{W}, R)\leftarrow \mathbf{Pert}(W)$}\label{subsubsec:pert}
The server initializes the global model at the beginning. Subsequently, it samples $K$ ($K \leq N$) clients and sends the global model to them for local training. As the global model $W=\{W^{(l)}\in \mathbb{R}^{n_l\times n_{l-1}}\}^{L}_{l=1}$ is a valuable asset, the server must perturb it before distribution. Specifically, the server randomly selects one-time-used noises $R$ for different iterations and perturbs the model parameter $W$ as follows:
  \begin{itemize}
  \item Select $L-1$ random multiplicative noisy vectors $\bm r^{(l)}=\{\bm r_{1}^{(l)}, \ldots, \bm r_{n_{l}}^{(l)}\} \in \mathbb{R}^{n_l}_{> 0 }$ for $1\leq l\leq L-1$, an additive noisy vector $\bm r^{a}=\{\bm r_{1}^{a}, \ldots, \bm r_{n_{L}}^{a}\}\in \mathbb{R}^{n_L}$ and a random coefficient vector $\bm \gamma=\{\bm \gamma_{1}, \ldots, \bm \gamma_{n_{L}}\} \in \mathbb{R}^{n_L}$. Then, generate the secret one-time-used model perturbation noisy matrices $R=(\{R^{(l)}\}^{L}_{l=1}, R^{a})$ as:
     \begin{eqnarray}
        R^{(l)}_{ij}&=&\left\{
        \begin{aligned}
        &\bm r^{(1)}_i ,  ~~~~~~~~\textnormal{ when } l=1 \\
        & \bm r^{(l)}_i / \bm r^{(l - 1)}_j, \textnormal{ when } 2\leq l \leq L-1 \\
        & 1 / \bm r^{(L - 1)}_j,  ~~\textnormal{ when } l=L
        \end{aligned}\right.  \label{eq:para_cons1} \\
        R^{a}_{ij}& =& \bm \gamma_i \cdot \bm r^{a}_{i}, \label{eq:para_cons2}
    \end{eqnarray}
    where $i\in [1, n_{l}]$ and $j\in [1,n_{l-1}]$ in Eq. \eqref{eq:para_cons1}, and $i\in [1, n_{L}]$ and $j\in [1,n_{L-1}]$ in Eq. \eqref{eq:para_cons2}.
  \item Generate the perturbed global model parameters $\widehat{W}=\{\widehat{W}^{(l)}\in \mathbb{R}^{n_l\times n_{l-1}}\}^{L}_{l=1}$ as: 
   \begin{equation}\label{eq:para_cons}
\widehat{W}^{(l)}=\left\{
\begin{aligned}
&  R^{(l)} \circ W^{(l)} ,  \textnormal{ for } 1\leq l\leq L-1;\\
& R^{(l)}  \circ W^{(l)} + R^{a},  \textnormal{ for } l =L,
\end{aligned}\right.
\end{equation}
where $\circ$ denotes the Hadamard product.
\end{itemize}

\subsubsection{Model Training $\{\nabla F(\widehat{W}^{(l)}, D_{k}), \bm{\Psi}_{k}^{(l)}, \bm \Phi_{k}^{(l)}\}^{L}_{l=1}\leftarrow \mathbf{Train}(\widehat{W}, D_{k})$}\label{subsubsec:MP_training}
Similar to the normal local model training in FL \cite{McMahanMRHA17}, each client $k$ runs the Stochastic Gradient Descent (SGD) method with its local dataset $D_{k}$ and the received perturbed global model $\widehat{W}=\{\widehat{W}^{(l)}\}^{L}_{l=1}$ to get the perturbed local gradients $\{\nabla F(\widehat{W}^{(l)}, D_{k})\}^{L}_{l=1}$. As shown in \cite{0003FFSTXL22}, the relationship between the perturbed local gradients and the true local gradients is established in the following theorem.

   \begin{theorem}[\cite{0003FFSTXL22}]\label{the:gradient}
    The perturbed local gradients $\{\nabla F(\widehat{W}^{(l)}, D_{k})\}^{L}_{l=1}$ and the true local gradients $\{\nabla F(W^{(l)}, D_{k})\}^{L}_{l=1}$ satisfy: for any $1 \leq l \leq L$,
\begin{equation*}
    \nabla F(W^{(l)}, D_{k})=R^{(l)} \circ \left(\nabla F(\widehat{W}^{(l)}, D_{k})-\bm{r}^{\top} \bm \Psi_{k}^{(l)}+ \upsilon \bm \Phi_{k}^{(l)}\right),
\end{equation*}
where $\bm{r}=\bm \gamma \circ \bm r^{a}$ and $\bm{r}^{\top}$ is its transpose, $\upsilon=\bm{r}^{\top}\bm{r}$, $\bm{\Psi}_{k}^{(l)}\in \mathbb{R}^{n_l\times n_{l-1}} $ and $\bm \Phi_{k}^{(l)}\in \mathbb{R}^{n_l\times n_{l-1}} $ are two noises which can be directly computed by the client $C_{k}$ (see \cite{0003FFSTXL22} for details of $\bm{\Psi}_{k}^{(l)}$ and $\bm \Phi_{k}^{(l)}$).
\end{theorem} 

Finally, each client $k$ sends $\{\nabla F(\widehat{W}^{(l)}, D_{k})\}^{L}_{l=1}$ and $\{\bm \Psi_{k}^{(l)}, \bm \Phi_{k}^{(l)}\}^{L}_{l=1}$ to the server.

\subsubsection{Aggregated Gradient Recovery $\{\nabla F(W^{(l)})\}^{L}_{l=1}\leftarrow \mathbf{Rec}(\{\{\nabla F(\widehat{W}^{(l)}, D_{k}),  \bm \Psi_{k}^{(l)}, \bm \Phi_{k}^{(l)}\}^{L}_{l=1}\}^{K}_{k=1}, R)$}\label{subsubsec:recover}
 After receiving the perturbed local gradients and the corresponding two noises from $K$ clients, the server aggregates them and recovers the true aggregated gradient with the secret one-time-used model perturbation noisy matrices $R$ as:
\begin{itemize}
    \item Aggregate $\{\{\nabla F(\widehat{W}^{(l)}, D_{k}),  \bm \Psi_{k}^{(l)}, \bm \Phi_{k}^{(l)}\}^{L}_{l=1}\}^{K}_{k=1}$ as
     \begin{equation*}
        \left\{
            \begin{aligned}
           & \nabla F(\widehat{W}^{(l)})= \frac{1}{K} \sum^{K}_{k=1}\nabla F(\widehat{W}^{(l)}, D_{k}) \\
           & \bm \Psi^{(l)}=  \frac{1}{K} \sum^{K}_{k=1}\bm \Psi_{k}^{(l)} \\
           & \bm \Phi^{(l)}= \frac{1}{K} \sum^{K}_{k=1}\bm \Phi_{k}^{(l)}   
            \end{aligned}\right. \textnormal{ for }  1 \leq l \leq L
    \end{equation*}
    \item According to Theorem \ref{the:gradient}, recover the true aggregated gradient $\{\nabla F(W^{(l)})\}^{L}_{l=1}$ with $R$ as: for $1 \leq l \leq L$,
    \begin{equation*}\label{eq:recovery0}
    \nabla F(W^{(l)})=R^{(l)} \circ (\nabla F(\widehat{W}^{(l)})-\bm{r}^{\top} \bm \Psi^{(l)}+ \upsilon \bm \Phi^{(l)} )
    \end{equation*}
     \item Update the current global model with the pre-set learning rate $\xi$ as
    \begin{equation*}
        W^{(l)}\leftarrow W^{(l)}-\xi \nabla F(W^{(l)}), \textnormal{ for }  1 \leq l \leq L
    \end{equation*}
\end{itemize}
The server and clients iteratively execute the above steps until the model converges or reaches the specified number of iterations.

\begin{theorem}[Privacy of \cite{0003FFSTXL22}]\label{theo:privacy_MP}
    Given the perturbed global model $\widehat{W}$, each semi-trusted client $k$ learns nothing about the true global model $W$ (including true local gradients $\{\nabla F(W^{(l)}, D_{k})\}^{L}_{l=1}$), except for the information deriving from its input $D_{k}$.
\end{theorem}

\begin{remark}
   The model perturbation method \cite{0003FFSTXL22} enables clients to perform the training program on the perturbed model $\widehat{W}$ with minimal computational overhead, offering an efficient approach to protecting the privacy of the server’s global model. However, it does not consider the privacy of clients. Particularly, the server can recover each client's true local gradients $\{\nabla F(W^{(l)}, D_{k})\}^{L}_{l=1}$ according to Theorem \ref{the:gradient}. Therefore, it fails to meet the requirements for bidirectional privacy preservation.
\end{remark}

\subsection{Problem Statement}\label{sec:problem}
We introduce our threat model and corresponding design goals, followed by a discussion of the key problems.

\subsubsection{Threat Model and Design Goals}
As illustrated in Fig. \ref{pipeline}, we assume that both clients and the server are semi-trusted, which means that they honestly follow the protocol but may infer the data privacy of others. 
\begin{itemize}
    \item The semi-trusted server may recover each client's training data from the received local gradients by launching reconstruction attacks \cite{ZhuLH19}.
    \item Semi-trusted clients seek to obtain the server's true global model and infer other clients' data privacy from the received data by performing membership inference attacks \cite{MelisSCS19, NasrSH19}. 
\end{itemize}
Drawing from the above threat model, our proposed scheme aims to achieve the following objectives:
\begin{itemize}
    \item \emph{Bidirectional privacy-preservation}: the server cannot access the gradients of individual clients, and similarly, the clients cannot obtain the global model.
    \item \emph{High performance}: the efficiency and model accuracy should be as close as possible to that of FL without privacy protection, e.g., FedAvg \cite{McMahanMRHA17}, to facilitate practical application.
\end{itemize}

\subsubsection{Problem Formulation}\label{subsubsec:prob}
Due to their high computation and communication efficiency, the combination of model perturbation and DP methods shows great potential for achieving bidirectional privacy preservation in FL. \emph{Therefore, the challenge lies in effectively integrating them}.

Next, we will provide a detailed analysis of the problems that need to be overcome when combining these two methods. Without loss of generality, suppose that we directly apply the original model perturbation method to protect the privacy of the global model and each client $k$ directly adds DP-noises $\bm\eta_{k}^{(l)} \sim \mathcal{N} (0, \sigma^2)$ to the perturbed local gradients $ \nabla F(\widehat{W}^{(l)}, D_{k})$\footnote{Note that from Theorem \ref{theo:privacy_MP}, each client $k$ cannot know the true local gradients, and so it can only add noises to the perturbed local gradients.}. Then, each client will generate the double-perturbed local gradients $\{\nabla \widehat{F}(\widehat{W}^{(l)}, D_{k})\}^{L}_{l=1}$ as 
\begin{equation} \label{UPDP}
    \nabla \widehat{F}(\widehat{W}^{(l)}, D_{k})= \nabla F(\widehat{W}^{(l)}, D_{k}) + \bm{\eta}_{k}^{(l)}, 
\end{equation}
where $1\leq l \leq L$, $\bm{\eta}_{k}^{(l)}$ is the DP-noisy matrix and the constraint will be analyzed in the subsequent sections. 

Similar to Section \ref{subsubsec:MP_training}, each client $k$ sends $\{\nabla \widehat{F}(\widehat{W}^{(l)}, D_{k}), \bm \Psi_{k}^{(l)}, \bm \Phi_{k}^{(l)}\}^{L}_{l=1}$ to the server. Based on Theorem \ref{the:gradient}, the server recovers $k$'s local gradients as 
\begin{flalign}\label{eq:recovery}
\nonumber   &\nabla \widehat{F}(W^{(l)}, D_{k})= R^{(l)} \circ \left( \nabla \widehat{F}(\widehat{W}^{(l)}, D_{k}) - \bm{r}^{\top} \Psi_{k}^{(l)} + \upsilon \bm \Phi_k^{(l)}\right)\\
 \nonumber    &=R^{(l)} \circ \left( \nabla F(\widehat{W}^{(l)}, D_{k}) + \bm\eta_{k}^{(l)}- \bm{r}^{\top} \Psi_{k}^{(l)} + \upsilon \bm \Phi_k^{(l)}\right)\\
    &=\nabla F(W^{(l)}, D_{k}) +  R^{(l)} \circ \bm\eta_{k}^{(l)}, \textnormal{ for }1\leq l \leq L.
\end{flalign}
From Eq. \eqref{eq:recovery}, we can see that to ensure the privacy of local gradients, $R^{(l)} \circ \bm\eta_{k}^{(l)}$ should follow the Gaussian distribution, i.e., $R^{(l)} \circ \bm\eta_{k}^{(l)}\sim\mathcal{N} (0, \sigma^2)$\footnote{Note that similar to many DP-based federated learning works \cite{AbadiCGMMT016}, we need to perform norm clipping operation to bound the gradients to a given range for determining the sensitivity $S_f$ defined in Eq. \eqref{eq:sensitive}. Specifically, one of the most popular norm clipping method is presented in \cite{AbadiCGMMT016}, which is defined as $\nabla \bar{F}(W^{(l)}, D_{k}) \leftarrow \nabla F(W^{(l)}, D_{k})/\max(1, \frac{\|\nabla F(W^{(l)}, D_{k})\|_{2}}{B})$ for a clipping threshold $B$. In this setting, the corresponding sensitivity is $S_f=B$. Since the primary focus of our work is to devise a noise addition method that satisfies differential privacy, the sensitivity is directly set to $S_f =B=1$ for simplicity.}, thereby ensuring that $\nabla \widehat{F}(W^{(l)}, D_{k})$ satisfies DP.

Recall that $R^{(l)}$ and $\bm\eta_{k}^{(l)}$ are two independent random variables generated by the server and the client $k$, respectively. Based on \cite{D2010}, if $\bm\eta_{k}^{(l)}$ follows the Gaussian distribution, then $R^{(l)} \circ \bm\eta_{k}^{(l)}$ cannot follow the Gaussian distribution, which demonstrates that $\nabla \widehat{F}(W^{(l)}, D_{k})$ in Eq. \eqref{eq:recovery} cannot satisfy DP mechanism. Consequently, the first problem to be addressed is outlined below.
\begin{problem}\label{Prob1}
    How to constrain the server's $R^{(l)}$ and each client $k$'s $\bm \eta_{k}^{(l)}$ to satisfy the bidirectional privacy-preservation, i.e., ensuring that the privacy of the global model $W$ satisfies Theorem \ref{theo:privacy_MP} and $R^{(l)}\circ \bm\eta_{k}^{(l)} \sim \mathcal{N} (0, \sigma^2)$.
\end{problem}

Next, if we successfully address Problem \ref{Prob1}, then the mechanism in Eq. \eqref{eq:recovery} is essentially the local differential privacy (LDP) \cite{KairouzBR16}, which allows each client to randomize its input (i.e., true local gradients $\nabla F(W^{(l)}, D_{k})$) locally before sending it to the server for aggregation. According to Eq. \eqref{eq:recovery}, the aggregated result is deduced as 
    \begin{flalign}\label{eq:aggre}
      \nonumber & \nabla \widehat{F}(W^{(l)})  =\frac{1}{K}\sum^{K}_{k=1}\nabla \widehat{F}(W^{(l)}, D_{k})\\
       & ~~~~=\frac{1}{K}\sum^{K}_{k=1}\nabla F(W^{(l)}, D_{k})+\frac{1}{K}R^{(l)} \circ \sum^{K}_{k=1} \bm{\eta}_{k}^{(l)}, 
    \end{flalign}
    where $1 \leq l \leq L$. As demonstrated in \cite{abs-2006-07218, ErlingssonFMRTT19}, the best possible error for the estimated average (i.e., the aggregation in FL) with $K$ clients is a factor of $O(\sqrt{K})$ larger than in the centralized model of DP (CDP) where a trusted curator aggregates data in the clear and perturbs the output. As a result, LDP is generally intractable in FL with a large number of clients \cite{abs-2006-07218}. Therefore, the second problem to be addressed is shown below.
\begin{problem}\label{Prob2}
    Given the privacy budget, how to minimize $R^{(l)}\circ \bm\eta_{k}^{(l)}$ to make the aggregated result match the accuracy of the CDP, thereby ensuring the utility of FL.
\end{problem}

\begin{figure*}[t]
\begin{center}
\includegraphics[width=0.7\textwidth]{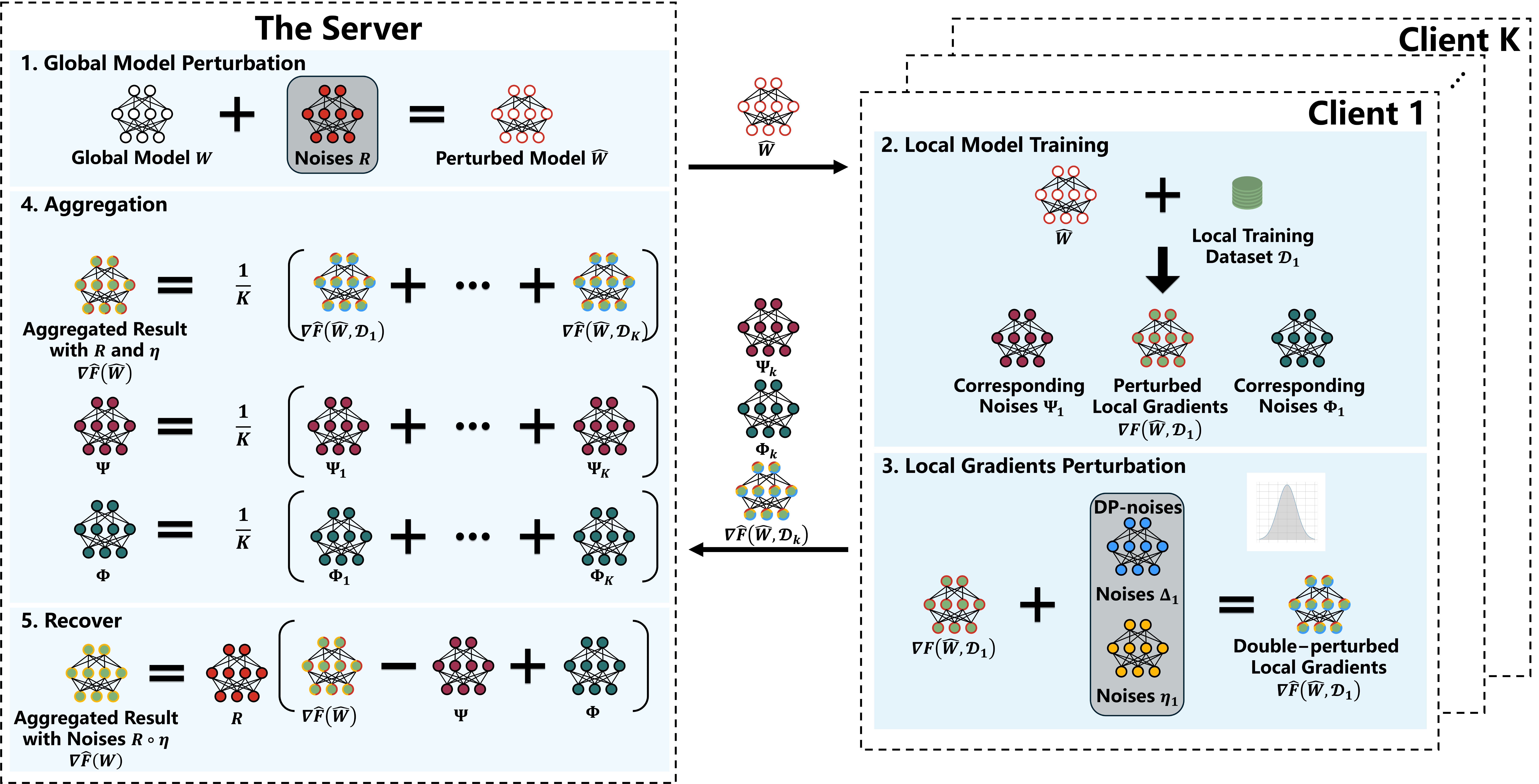}
\end{center}
  \caption{The architecture overview of the proposed MP-DP, where the grey shaded area is our focus.}
\label{pipeline}
\end{figure*}

\section{Our Proposed Scheme}\label{sec:MP_DP}
In this section, we describe the details of our scheme. Specifically, we first address the two problems summarized in Section \ref{subsubsec:prob} and then give the overall process of our scheme, as shown in Fig. \ref{pipeline}.

\subsection{The Solution for Problem \ref{Prob1}}\label{sub:Prob1}
As analyzed in Section \ref{subsubsec:prob}, since $R^{(l)}$ and $\bm\eta_{k}^{(l)}$ are random variables generated by the server and client, respectively, the result $R^{(l)} \circ \bm\eta_{k}^{(l)}$ in Eq. \eqref{eq:recovery} is the multiplication of two independent random variables, where any $i$-th row and $j$-th column of $R^{(l)} \circ \bm\eta_{k}^{(l)}$ is defined as
\begin{eqnarray}\label{eq:product0}
R^{(l)}_{ij} \cdot \bm\eta_{kij}^{(l)} &=&\left\{
\begin{aligned}
 & \bm\eta_{kij}^{(1)} \bm r^{(1)}_i,  \textnormal{ if } l=1 \\
& \frac{\bm\eta_{kij}^{(l)} \bm r^{(l)}_i}{\bm r^{(l - 1)}_j}, \textnormal{ if } 2\leq l \leq L-1 \\
 & \frac{\bm\eta_{kij}^{(L)}}{\bm r^{(L - 1)}_j},  \textnormal{ if } l=L
\end{aligned}\right.  
\end{eqnarray}
where $R^{(l)}_{ij}$ and $\bm{\eta}_{kij}^{(l)}$ are the elements in the $i$-th row and $j$-th column of $R^{(l)}$ and $\bm\eta_{k}^{(l)}$, respectively.

To guarantee $R^{(l)}\circ \bm\eta_{k}^{(l)} \sim \mathcal{N}(0, \sigma^2)$, intuitively, we should decompose the Gaussian variable into the product of several independent random variables. We present this decomposition in Theorem \ref{defactorization}, whose proof is given in Appendix \ref{appendix_theorem3}.

\begin{theorem}\label{defactorization}
Let $Z$ be the normal random variable $\mathcal{N}(0, \sigma^2)$, there exist $m$ independent identically distributed (i.i.d.) random variables $X_{1}, X_{2}, \dots,X_{m}$ and $n$ i.i.d. random variables $Y_{1}, Y_{2}, \dots,Y_{n}$ such that
\begin{equation}\label{eq:multiplic}
	Z\overset{D}{=} X_{1}X_{2}\cdots X_{m}Y_{1}Y_{2}\cdots Y_{n},
\end{equation}
where $\overset{D}{=}$ denotes the equality in distribution. More specifically, for any  $1 \leq i \leq m$ and $1 \leq j \leq n$, $X_{i}$ and $Y_{j}$ follow the distributions $\mathcal{DN}^{*}(m)$ and $\mathcal{DN}(\sigma, n)$, respectively, as defined below:
\begin{equation*}
    \left\{
        \begin{aligned}
            & \mathcal{DN}^{*}(m)\overset{D}{=}\beta \exp\Big\{-\sum_{\ell=1}^{\infty}\left[\frac{G_{1/m,\ell}}{2 \ell+1}-\frac{1}{2m} \ln \left(1+\frac{1}{\ell}\right)\right]\Big\},\\
            & \mathcal{DN}(\sigma, n)\overset{D}{=} \beta \exp\Big\{\frac{\ln(\sqrt{2}\sigma)}{n}-G_{1/n}\Big\}, \\
        \end{aligned}
    \right.
\end{equation*}
where $G_{1/n}, G_{1/m,1}, \dots$ are i.i.d. random variables, each with the gamma distribution $\operatorname{Gamma}(1/n,1)$ with shape parameter $1/n$ and scale parameter $1$ and $\beta$ is a Rademacher random variable so that $\Pr[\beta=1]=\Pr[\beta=-1]=1/2$.
\end{theorem}

 Although Theorem \ref{defactorization} perfectly ensures that the product of multiple random variables follows $\mathcal{N}(0, \sigma^2)$, as given in Challenge \ref{challenge1}, directly applying it to Eq. \eqref{eq:product0} still cannot overcome $R^{(l)}\circ \bm\eta_{k}^{(l)} \sim \mathcal{N}(0, \sigma^2)$ in Problem \ref{Prob1}. 
 
\begin{challenge}\label{challenge1}
 The key is that the server-side random variable $\bm{r}_i^{(l)}$ ($1 \leq l \leq L - 1$) is used for the $l$-th layer, and its reciprocal $1/\bm{r}_i^{(l)}$ is also used for the $l+1$-th layer. For example, if $\bm{r}_i^{(l)}\sim \mathcal{DN^{*}}(2)$ and $\bm\eta_{kij}^{(l)} \sim \mathcal{DN}(\sigma, 1)$ for $2 \leq l \leq L - 1$, then $1/\bm{r}_i^{(l)}$ cannot conform $\mathcal{DN^{*}}(2)$ according to Theorem \ref{defactorization}. Thus, we cannot guarantee $R_{ij}^{(l)} \cdot \bm\eta_{kij}^{(l)}=\bm\eta_{kij}^{(l)} \bm r^{(l)}_i/\bm r^{(l - 1)}_j \sim \mathcal{N}(0, \sigma_{\eta}^2)$
\end{challenge}

Therefore, we solve the above challenge by introducing $L-1$ transitional layers. More specifically, for the original global model $\{W_{\mathbf{org}}^{(l)}\}_{l=1}^L$ with $L$ layers, we expand it into $2L-1$ layers, where the parameters of each layer are set as
\begin{equation}
   \left\{
\begin{aligned}
& W^{(2l-1)} \leftarrow W_{\mathbf{org}}^{(l)} , \textnormal{ if } 1\leq l \leq L \\
& W^{(2l)} \leftarrow \mathcal{I}^{(l)},  \textnormal{ if } 1\leq l \leq L-1
\end{aligned}  \right.\label{eq:revisedW}
\end{equation} 
where $\mathcal{I}^{(l)}$ is the $n_{l}\times n_{l}$ dimensional identity matrix and entirely different from the preceding layers. In addition to keeping random vectors $\{\{\bm r^{(l)} \in \mathbb{R}^{n_{l}}_{> 0 } \}_{l=1}^{L-1}, \bm \gamma, \bm r^{(a)}\}$ for the odd layers as stated in Section \ref{subsubsec:pert}, the server also generates random vectors $\{\bm{s}^{(l)} \in \mathbb{R}^{n_{l}}_{> 0 }\}^{L-1}_{l=1}$ for $L-1$ even layers. Based on Theorem \ref{defactorization}, these random vectors should satisfy the following constraints: for any $i \in[1, n_{l}]$, the $i$-th elements of vectors $\bm r^{(l)}$ and $\bm{s}^{(l)}$ are
\begin{equation}\label{constraint:r}
   \left\{
\begin{aligned}
&  \bm r_i^{(1)}, \bm{s}^{(L-1)}_{i} \sim \mathcal{DN}^{*}(1) \\
&  \bm r_i^{(l)} \sim\mathcal{DN}^{*}(2), \textnormal{ if } 1 < l \leq L-1 \\
&  \bm{s}^{(l)}_{i}\sim\mathcal{DN}^{*}(2),\textnormal{ if } 1 \leq l < L-1\\
\end{aligned}\right.  
\end{equation}
According to the above constraints, each element of $R^{(l)}$ in Eq. \eqref{eq:para_cons1} is revised as:
\begin{equation}\label{eq:Rrevised1}
   \left\{
\begin{aligned}
&   R^{(1)}_{ij}=\bm r^{(1)}_i \textnormal{ and }  R^{(2L-1)}_{ij}=\bm{s}_{j}^{(L-1)}\\
&  R^{(2l-1)}_{ij}=\bm r_i^{(l)}\bm{s}^{(l-1)}_{j}, \textnormal{ if $1 < l \leq L-1$} \\
&  R^{(2l)}_{ij}=\left\{\begin{aligned}
        &\frac{1}{\bm{s}^{(l)}_{i}\bm r_i^{(l)}}, \textnormal{ if $i=j$}\\
        & 0, ~~~~~~~~\textnormal{ if $i\neq j$} \\
\end{aligned}\right.,\textnormal{if $1 \leq l \leq L-1$}\\
\end{aligned}\right. 
\end{equation}
From Eqs. \eqref{constraint:r} and \eqref{eq:Rrevised1}, we can see that the random variables $\bm r^{(l)}_{i}$ and $\bm{s}^{(l)}_{i}$ are used in the odd-indexed layers, which follow the distribution $\mathcal{DN}^{*}(m)$ in Theorem \ref{defactorization}. There corresponding reciprocals $1/\bm r^{(l)}_{i}$ and $1/\bm{s}^{(l)}_{i}$ are used in the even-indexed layers. Since the model parameters of even-indexed layers are independent of the true parameters, we do not need $1/\bm{s}^{(l)}_{i}$ and $1/\bm r_i^{(l)}$ to satisfy the distribution in Theorem \ref{defactorization}. Here, we call newly inserted layers as \textbf{transitional layers} as they only serve to address Challenge \ref{challenge1}. 

Similar to Eq. \eqref{eq:para_cons}, the perturbed global model $\widehat{W}=\{\widehat{W}^{(l)}\}^{2L-1}_{l=1}$ is generated as
          \begin{equation}\label{eq:para_consRevise}
                \widehat{W}^{(l)}=\left\{
                    \begin{aligned}
                        &  R^{(l)} \circ W^{(l)} ,  \textnormal{ if } 1\leq l\leq 2L-2,\\
                        & R^{(l)}  \circ W^{(l)} + R^{(a)},  \textnormal{ if } l =2L-1,
                    \end{aligned}\right.
            \end{equation}
where $R^{(a)}_{ij}= \bm \gamma_i \cdot \bm r^{(a)}_{i}$ is the same as in Eq. \eqref{eq:para_cons2}. 

Next, we present the following theorem to demonstrate that the perturbed local gradients computed from the expanded perturbed global model $\widehat{W}=\{\widehat{W}^{(l)}\}^{2L-1}_{l=1}$  can be used to update the original model $ \{ W_{\mathbf{org}}^{(l)}\}^{L}_{l=1}$. The proof of Theorem \ref{theorem:relation} is given in Appendix \ref{appendix_theorem4}.

\begin{theorem}\label{theorem:relation}
     The odd layers' perturbed local gradients $\{\nabla F(\widehat{W}^{(2l-1)}, D_{k})\}^{L}_{l=1}$ from the expanded perturbed global model $\{\widehat{W}^{(l)}\}^{2L-1}_{l=1}$ and true local gradients $\{\nabla F( W_{\mathbf{org}}^{(l)}, D_{k})\}^{L}_{l=1}$ computed from the original global model $ \{ W_{\mathbf{org}}^{(l)}\}^{L}_{l=1}$ satisfy
     \begin{flalign*}
          \nabla F( W_{\mathbf{org}}^{(l)}, D_{k})=&R^{(2l-1)} \circ \big(\nabla F(\widehat{W}^{(2l-1)}, D_{k})-\bm{r}^{\top} \bm \Psi_{k}^{(2l-1)}\\
          &+\upsilon \bm \Phi_{k}^{(2l-1)}\big) \textnormal{ for }1 \leq l \leq L,
     \end{flalign*}
     where $\bm{r}=\bm \gamma \circ \bm r^{(a)}$ and $\bm{r}^{\top}$ is its transpose, $\upsilon=\bm{r}^{\top}\bm{r}$, $\bm{\Psi}_{k}^{(2l-1)}$ and $\bm \Phi_{k}^{(2l-1)}$ are two noises which are directly computed by the client $k$ (see Appendix \ref{appendix_theorem4} for the details).
\end{theorem}
\begin{remark}\label{remark-theorem4}
Theorem \ref{theorem:relation} indicates that the local gradients of the odd-indexed layers obtained from the expanded global model are identical to those from the original model (i.e., $\nabla F( W^{(2l-1)}, D_{k})=\nabla F( W_{\mathbf{org}}^{(l)}, D_{k})$ for $1 \leq l \leq L$). Therefore, only the odd-indexed layers' parameters corresponding to the expanded global model $\{W^{(l)}\}^{2L-1}_{l=1}$, like $\{\nabla F(\widehat{W}^{(2l-1)}, D_{k}), \bm \Psi_{k}^{(2l-1)}, \bm \Phi_{k}^{(2l-1)}\}^{L}_{l=1}$, are required to update the original global model $\{W_{\mathbf{org}}^{(l)}\}^{L}_{l=1}$. Thus, in the following sections, we omit the description of the even-indexed transition layers and focus on the implementation of the odd-indexed layers, including adding DP noise, model updates, and privacy analysis.
\end{remark}

\begin{algorithm}[htbp]
    \caption{$\mathbf{MP\_Server}$} \label{MP-Server}
    \begin{algorithmic}[1]
        \REQUIRE Global model $\{ W_{\mathbf{org}}^{(l)}\}^{L}_{l=1}$ 
        \ENSURE Perturbed global model $\widehat{W}=\{\widehat{W}^{(l)}\}^{2L-1}_{l=1}$ and secret one-time-used noises $R=(\{R^{(l)}\}^{2L-1}_{l=1}, R^{(a)})$ 
        \STATE{Expand the number of neural network layers from $L$ to $2L-1$, i.e., $\{ W_{\mathbf{org}}^{(l)}\}^{L}_{l=1} \rightarrow \{W^{(l)}\}^{2L-1}_{l=1} $, where the parameters of each layer are set as Eq. \eqref{eq:revisedW}.}
        \STATE {For $1\leq l < L$, draw random vectors $\bm{r}^{(l)}$ and $\bm{s}^{(l)}$ according to Eq. \eqref{constraint:r}. Besides, draw $\bm \gamma, \bm r^{(a)} \in \mathbb{R}^{n_{L}}$ following in Section \ref{subsubsec:pert}.}
        \STATE{Generate $(\{R^{(l)}\}^{2L-1}_{l=1}, R^{(a)})$ from Eqs. \eqref{eq:Rrevised1} and \eqref{eq:para_cons2}.}
        \STATE{Generate $\widehat{W}=\{\widehat{W}^{(l)}\}^{2L-1}_{l=1}$ based on Eq. \eqref{eq:para_consRevise}.}
          \RETURN{$\widehat{W}=\{\widehat{W}^{(l)}\}^{2L-1}_{l=1}$ and $R=(\{R^{(l)}\}^{2L-1}_{l=1}, R^{(a)})$ }
    \end{algorithmic}
    \label{alg:rndphase}
\end{algorithm}

Similarly, client-side random variables $\{\bm\eta_{k}^{(2l-1)}\}_{l=1}^{L}$ need to follow the constraint below: 
\begin{equation}\label{constraint:epsilon}
     \bm\eta_{kij}^{(2l-1)} \sim \mathcal{DN}(\sigma_{\eta}, 1), \textnormal{ for } 1\leq l \leq L,
\end{equation}
where $\sigma^2_{\eta}$ is the variance for the random variables $\{\bm\eta_{k}^{(2l-1)}\}_{l=1}^{L}$. Here we redefine the variance with a subscript $\eta$ to distinguish it from the variance of another variable to be introduced later. Combined with Eqs. \eqref{constraint:r}, \eqref{eq:Rrevised1} and \eqref{constraint:epsilon}, we derive Theorem \ref{distribution}, whose proof is given in Appendix \ref{appendix_theorem45}.
\begin{theorem}\label{distribution}
For 1 $\leq l \leq$ $L$, if $R^{(2l-1)}$ and $\bm\eta_{k}^{(2l-1)}$ satisfy Eqs. \eqref{constraint:r} and \eqref{constraint:epsilon}, respectively, then $R^{(2l-1)} \circ \bm\eta_{k}^{(2l-1)} \sim\mathcal{N}(0, \sigma_{\eta}^2)$.
\end{theorem}

\emph{So far, the first main contribution of our scheme is to design a novel global model perturbation method ($\mathbf{MP\_Server}$) and constrain the client-side random variables skillfully to address Problem \ref{Prob1}.} The overall pipeline for $\mathbf{MP\_Server}$ is shown in Algorithm \ref{MP-Server}.

\subsection{The Solution for Problem \ref{Prob2}}\label{subsec:Prob2}
Based on Eq. \eqref{eq:aggre}, the aggregated result is computed as
\begin{equation*}
    \nabla \widehat{F}(W^{(2l-1)})=\nabla F(W^{(2l-1)})+\frac{1}{K}R^{(2l-1)} \circ \sum^{K}_{k=1} \bm{\eta}_{k}^{(2l-1)},
\end{equation*}
where $\nabla F(W^{(2l-1)})=\frac{1}{K}\sum^{K}_{k=1} \nabla F(W^{(2l-1)}, D_{k})$ for $1\leq l \leq L$. We can see that $\nabla \widehat{F}(W^{(2l-1)})$ has expected value $\nabla F(W^{(2l-1)})$ and variance $\sigma^{2}_{\eta}/K$. To address \textbf{Problem \ref{Prob2}}, we want the amount of noise to be of the same order of magnitude as that required to protect the true aggregated result using CDP, i.e., $\sigma^{2}_{\eta}=O(1/K)$ \cite{DworkR14, abs-2006-07218}. This setting will result in less noise being added to local gradients, thereby reducing the protective capacity of LDP.

To address this, we compensate the independent Gaussian noise with pairwise-correlated noises that can eventually be canceled. The high-level idea is to have each client $k$ mask its local gradient by adding two different types of noises. The first is a sum of pairwise-correlated noises over the set of neighbors such that all correlated noises can be canceled out in the aggregated result. The second type of noise is an independent term $\{\bm\eta_{k}^{(2l-1)}\}_{l=1}^{L}$, which dose not cancel out in the aggregated result. \emph{These pairwise-correlated noises can be sufficiently large to make the variance $\sigma_{\eta}^{2}$ of the independent (non-canceling) noises as small as in CDP to keep the same utility level as CDP}.

Let the clients communicate over a network represented by a connected undirected graph $G=(C, E)$, where $C$ is the set of clients, and if clients $k$ and $v$ can exchange messages (i.e., negotiate pairwise-correlated noises), then they are neighbors denoted as $\{k, v\}\in E$. For a given client $k$, we denote by $N(k)=\{v: \{k, v\}\in E\}$ the set of its neighbors. The client $k$ first communicates with each of its neighbors $v \in N(k)$ so as to generate the correlated noise matrix $\bm\Delta^{(2l-1)}_{k, v} = - \bm\Delta^{(2l-1)}_{v, k}$ for $1\leq l \leq L$. Then, the client $k$ selects the independent noises $\{\bm\eta_{k}^{(2l-1)}\}_{l=1}^{L}$ based on Eq. \eqref{constraint:epsilon}. After that, the client $k$ adds these two types of noises into the perturbed gradients as 
\begin{flalign}\label{eq:double-pertube}
 \nonumber \nabla \widehat{F}(\widehat{W}^{(2l-1)}, D_{k})= &  \nabla F(\widehat{W}^{(2l-1)}, D_{k})+\sum_{v\in N(k)}\bm\Delta^{(2l-1)}_{k, v}\\ 
  &+\bm{\eta}_{k}^{(2l-1)}, \textnormal{ for }1\leq l \leq L
\end{flalign}
Similar to Eq. \eqref{eq:recovery}, with the above $\nabla \widehat{F}(\widehat{W}^{(2l-1)}, D_{k})$, the final recovered outputs for the server are deduced as 
\begin{flalign}
\nonumber    \nabla \widehat{F}(W^{(2l-1)}, D_{k})=& \nabla F(W^{(2l-1)}, D_{k})+R^{(2l-1)}\circ \big(\bm{\eta}_{k}^{(2l-1)} \\
    &+\sum_{v\in N(k)}\bm\Delta^{(2l-1)}_{k, v}\big), \label{eq:doublenoise}
\end{flalign}
where $1 \leq l \leq L$. Similar to \cite{abs-2006-07218}, for any $1 \leq l \leq L$, $R^{(2l-1)}\circ \bm\Delta^{(2l-1)}_{k, v}$ should also satisfy the standard Gaussian mechanism $\mathcal{N}(0, \sigma_{\Delta}^2)$, where $\sigma^{2}_{\Delta}$ is the variance for the random variables $\{\bm\Delta^{(2l-1)}_{k, v}\}^{L}_{l=1}$. Thus, based on Theorem \ref{defactorization}, each element in $i$-th row and $j$-th column of $\bm\Delta^{(2l-1)}_{k, v}$, denoted as $\bm\Delta^{(2l-1)}_{k, v}[ij]$, should satisfy 
\begin{equation} \label{constraint:Delta}
 \bm\Delta^{(2l-1)}_{k, v}[ij] \sim \mathcal{DN}(\sigma_{\Delta}, 1), \textnormal{ for }  1 \leq l \leq L,
\end{equation}
Combined with Eqs. \eqref{constraint:r}, \eqref{eq:Rrevised1} and \eqref{constraint:Delta}, we can derive the following theorem, whose proof is shown in Appendix \ref{appendix_theorem45}.
\begin{theorem}\label{distribution2}
For any 1 $\leq l \leq$ $L$, $R^{(2l-1)}\circ \bm\Delta^{(2l-1)}_{k, v}\sim \mathcal{N}(0, \sigma_{\Delta}^2)$ when $R^{(2l-1)}$ and $\bm\Delta^{(2l-1)}_{k, v}$ satisfy Eqs. \eqref{constraint:r} and \eqref{constraint:Delta}, respectively.
\end{theorem}

As shown in Section \ref{subsubsec:recover}, with the double-perturbed local gradients $\nabla \widehat{F}(\widehat{W}^{(2l-1)}, D_{k})$, the server computes the aggregated result as 
\begin{flalign}\label{eq:aggregated_gradient}
 \nonumber   &\nabla \widehat{F}(\widehat{W}^{(2l-1)}) = \frac{1}{K}\sum^{K}_{k=1}\nabla \widehat{F}(\widehat{W}^{(2l-1)}, D_{k}) \\
  \nonumber  &= \frac{1}{K} \sum^{K}_{k=1}\left(\nabla F(\widehat{W}^{(2l-1)}, D_{k})+\sum_{v\in N(k)}\bm\Delta^{(2l-1)}_{k, v} + \bm{\eta}_{k}^{(2l-1)}\right)\\
   &\stackrel{(a)}{=} \frac{1}{K} \sum^{K}_{k=1}\nabla F(\widehat{W}^{(2l-1)}, D_{k})+\frac{1}{K}\sum^{K}_{k=1} \bm{\eta}_{k}^{(2l-1)}, 
\end{flalign}
where Eq. $(a)$ comes by $\sum^{K}_{k=1}\sum_{v\in N(k)}\bm\Delta^{(2l-1)}_{k, v}=0$ as $\bm\Delta^{(2l-1)}_{k, v}=-\bm\Delta^{(2l-1)}_{v, k}$ for $1 \leq l \leq L$. 

Similarly, two noisy items are also aggregated as
\begin{equation}\label{eq:aggregated_noises}
    \left\{
    \begin{aligned}
        & \bm \Psi^{(2l-1)}= \frac{1}{K} \sum^{K}_{k=1}\bm \Psi_{k}^{(2l-1)}; \\
    & \bm \Phi^{(2l-1)}=\frac{1}{K} \sum^{K}_{k=1} \bm \Phi_{k}^{(2l-1)}; 
    \end{aligned} \textnormal{ for } 1 \leq l \leq L.
    \right.
\end{equation}
Based on Eqs. \eqref{eq:aggregated_gradient} and \eqref{eq:aggregated_noises}, the aggregated result that only contains independent Gaussian noises, denoted as $\nabla \widehat{F}(W^{(2l-1)})$, can be deduced with the secret one-time-used noises $\{R^{2l-1}\}^{L}_{l=1}$:  for $1\leq l\leq L$,
  \begin{align}\label{noises_elimination}
 \nonumber  \nabla \widehat{F}(W^{(2l-1)})=&R^{(2l-1)} \circ \big(\nabla \widehat{F}(\widehat{W}^{(2l-1)}) - \bm{r}^{\top} \bm \Psi^{(2l-1)}  \\
 \nonumber & + \upsilon \bm \Phi^{(2l-1)}\big)\\
\stackrel{(b)}{=}& \nabla F(W^{(2l-1)})+R^{(2l-1)}\circ \bm{\eta}^{(2l-1)}, 
  \end{align}
where Eq. $(b)$ follows from the result in Eq. \eqref{eq:recovery}, $\nabla F(W^{(2l-1)})=\frac{1}{K} \sum^{K}_{k=1}\nabla F(W^{(2l-1)}, D_{k})$ and $\bm{\eta}^{(2l-1)}=\frac{1}{K}\sum^{K}_{k=1} \bm{\eta}_{k}^{(2l-1)}$.

From Eq. \eqref{noises_elimination}, we can see that the final aggregation result only contains the independent Gaussian noises $R^{(2l-1)}\circ \bm{\eta}^{(2l-1)}=R^{(2l-1)}\circ \frac{1}{K}\sum^{K}_{k=1} \bm{\eta}_{k}^{(2l-1)}$ with the variance $\sigma^{2}_{\eta}/K$. As discussed before, under the premise of ensuring the privacy of local gradients (i.e., LDP), the variance $\sigma_{\Delta}^{2}$ of pairwise-correlated noises should be sufficiently large to make the variance $\sigma_{\eta}^{2}$ of independent noises as small as in the CDP (i.e., $\sigma^{2}_{\eta}=O(1/K)$) to guarantee model accuracy. The detailed constraints on $\sigma^{2}_{\eta}$ and $\sigma_{\Delta}^{2}$ will be discussed based on privacy budget in Theorems \ref{thm:diffprivacy-complete} and \ref{thm:diffprivacy-random}.

\emph{Consequently, the second main contribution and innovation of our scheme is the elegantly designed mechanism for adding differential noises, which makes the aggregated result achieve the same accuracy as CDP while ensuring the privacy of local gradients}. The overall pipeline for generating double-perturbed local gradients (called $\mathbf{DDP\_Client}$) is shown in Algorithm \ref{DPclient}.

\begin{algorithm}[htbp]
    \caption{$\mathbf{DDP\_Client}$} \label{DPclient}
    \begin{algorithmic}[1]
        \REQUIRE Graph $G=(C, E)$, the variances $\sigma_{\eta}^2, \sigma_{\Delta}^{2}\in \mathbb{R}^+$, perturbed local gradients 
        $\{\nabla F(\widehat{W}^{(2l-1)}, D_{k})\}^{L}_{l=1}$.
        \ENSURE Double-perturbed gradients $\nabla \widehat{F}(\widehat{W}^{(2l-1)}, D_k)$ for $1 \leq l \leq L$. 
        \STATE {Draw random matrices $\{\bm{\eta}_{k}^{(2l-1)}\}^{L}_{l=1}$ based on Eq. \eqref{constraint:epsilon}}
        \STATE{Determine the set of the neighbors $N(k)$ from the graph $G=(C, E)$}
        \FORALL{$v \in N(k)$ and $k < v$}
         \STATE{Clients $k$ and $v$ draw $\{\bm\Delta^{(2l-1)}_{k, v}\}^{L}_{l=1}$ following Eq. \eqref{constraint:Delta} and set $\bm\Delta^{(2l-1)}_{v, k}=-\bm\Delta^{(2l-1)}_{k, v} $ for any $1\leq l\leq L$}
         \ENDFOR
        \STATE{Compute $\{\nabla \widehat{F}(\widehat{W}^{(2l-1)}, D_{k})\}^{L}_{l=1}$ based on Eq. \eqref{eq:double-pertube} } 
          \RETURN{$\{\nabla \widehat{F}(\widehat{W}^{(2l-1)}, D_{k})\}^{L}_{l=1}$}
    \end{algorithmic}\label{alg:rndphase}
\end{algorithm}

\begin{algorithm}[htbp]
    \caption{$\mathbf{MU\_Server}$} \label{MU-Server}
    \begin{algorithmic}[1]
        \REQUIRE Double perturbed local gradients and two noisy items $\{\{\nabla \widehat{F}(\widehat{W}^{(2l-1)}, D_{k}), \bm \Psi_{k}^{(2l-1)}, \bm \Phi_{k}^{(2l-1)}\}^{L}_{l=1}\}^{K}_{k=1}$, secret one-time-used noises $R=(\{R^{(l)}\}^{2L-1}_{l=1}, R^{(a)})$, current global model $\{ W_{\mathbf{org}}^{(l)}\}^{L}_{l=1}$ and learning rate $\xi$.
        \ENSURE Updated global model for the next iteration.
        \STATE{Aggregate $K$ double-perturbed local gradients and the corresponding noisy items based on Eqs. \eqref{eq:aggregated_gradient} and \eqref{eq:aggregated_noises}. }
        \STATE {Recover the aggregated result that only contains independent noises, i.e., $ \nonumber \nabla \widehat{F}(W^{(2l-1)})$, from Eq. \eqref{noises_elimination}.}
        \STATE{According to Theorem \ref{theorem:relation}, update the global model as
        \begin{equation*}
         W_{\mathbf{org}}^{(l)}\leftarrow  W_{\mathbf{org}}^{(l)}-\xi \nabla \widehat{F}(W^{(2l-1)}), \textnormal{ for } 1\leq l < L
      \end{equation*}
        } 
        \RETURN{Updated global model $\{ W_{\mathbf{org}}^{(l)}\}^{L}_{l=1}$}
    \end{algorithmic}
    \label{alg:rndphase}
\end{algorithm}

\begin{remark}
    In Algorithm \ref{DPclient}, each client $k$ should negotiate with its neighbors $N(k)$ to securely determine the correlated random variables $\{\bm\Delta^{(2l-1)}_{k, v}\}^{L}_{l=1}$ for any $v \in N(k)$. In a nutshell, for any $v \in N(k)$ and $k < v$, the client $k$ randomly selects $\{\bm\Delta^{(2l-1)}_{k, v}\}^{L}_{l=1}$ according to Eq. \eqref{constraint:Delta} and securely sends it to the client $v$ through the Diffie-Hellman key agreement (DH) \cite{DiffieH76} and Pseudo-random Generator (PRG) \cite{Yao82a}. The corresponding details can be referred to \cite{BonawitzIKMMPRS17}. However, the scheme in \cite{BonawitzIKMMPRS17} considers a complete graph $G=(C, E)$, i.e., for any two clients $k, v\in C$, $\{k, v\}\in E$ (all clients are neighbors). Then each client $k$ needs to interact with all $K-1$ remaining clients in $C$ to determine the correlated random variables, where $K$ is the number of clients in $C$. Obviously, this setting is not fully satisfactory from the practical perspective when the number of clients is large. On the contrary, our $\mathbf{DDP\_Client}$ algorithm considers a sparse network graph $G$ with random $n$ clients interconnected to accommodate the large-scale federated learning. The idea is to make each client $k$ select $n$ other clients uniformly at random among all clients in $C$, where $n\leq K-1$ and $n$ is the number of clients in the set $N(k)$. Then, the edge $\{u, v\}\in E$ is created if $u$ selected $v$ or $v$ selected $u$ (or both). This is known as a random $n$-out graph \cite{Bollobas11}.  Fig. \ref{fig_graph} gives an example of Algorithm \ref{DPclient} for the complete and random $n$-out graphs.
\end{remark}

\begin{figure*}[!t]
\centering
\includegraphics[width=1\textwidth]{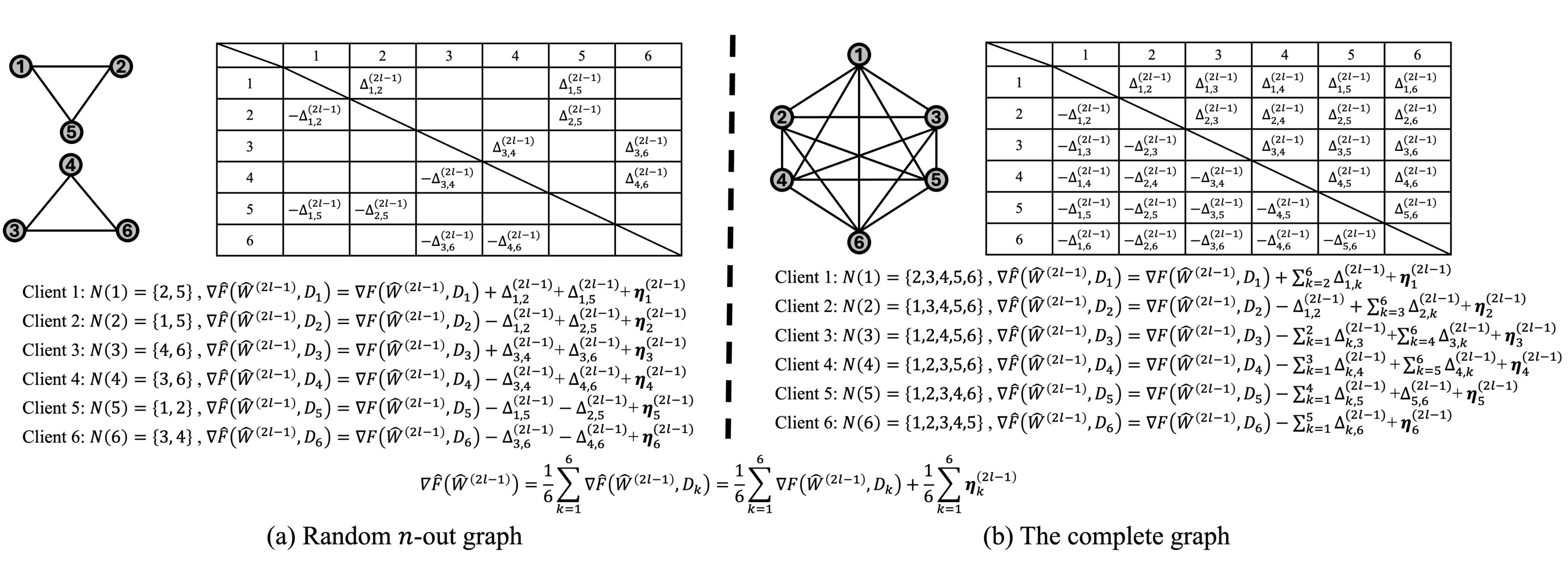}
  \caption{An example of Algorithm \ref{DPclient} for the complete and random $n$-out graphs, where $C=\{1,2,3,4,5,6\}$, $n=3$.}
\label{fig_graph}
\end{figure*}

\begin{algorithm*}[htbp]
    \caption{The overall framework of our scheme} \label{MP-DP}
    \begin{algorithmic}[1]
        \REQUIRE The variances $\sigma^2_{\eta}, \sigma^2_{\Delta}\in \mathbb{R}^+$, the number of rounds $T$, $N$ clients indexed by $\{1,2,\ldots, N\}$, the fraction $\rho$ of clients that perform local training on each round, the local training dataset $D_{k}$ of the client $k$, the pre-set learning rate $\xi$.
        \ENSURE The well-trained global model $W$. 
        \STATE{The server initializes the global model $W_{\mathbf{org\_1}}=\{W_{\mathbf{org\_1}}^{(l)}\}^{L}_{l=1}$}
        \FOR{each round $t=1,2,\ldots, T$}
        \STATE{Server executes: 
            \begin{itemize}
                \item $\left\{\widehat{W}_{t}=\left\{\widehat{W}^{(l)}_{t}\right\}^{2L-1}_{l=1}, R_{t}=\left(\left\{R_{t}^{(l)}\right\}^{2L-1}_{l=1}, R_{t}^{(a)}\right)\right\}\leftarrow\mathbf{MP\_Server}(W_{\mathbf{org\_t}})$
                \item Sample the set $S_{t}$ of $K$ clients from $\{1,\ldots, N\}$, where $K=\rho N$, and send $\widehat{W}_{t}=\left\{\widehat{W}^{(l)}_{t}\right\}^{2L-1}_{l=1}$ to clients in $S_{t}$
            \end{itemize}
        }
            \FOR{each client $k\in S_{t}$ \textbf{in parallel}}
                \STATE{$\left\{\nabla F(\widehat{W}_{t}^{(l)}, D_{k}), \bm{\Psi}_{k}^{(l)}, \bm \Phi_{k}^{(l)}\right\}^{2L-1}_{l=1}\leftarrow \verb"LMT"\left(\widehat{W}_{t}, D_{k}\right)$}
                \STATE{
                    $\left\{\nabla \widehat{F}(\widehat{W}_{t}^{(2l-1)}, D_{k})\right\}^{L}_{l=1}\leftarrow\mathbf{DDP\_Client}\left(G_{t}=(S_{t}, E_{t}), \left\{\nabla F(\widehat{W}_t^{(2l-1)}, D_{k})\right\}^{L}_{l=1}, \sigma^2\right)$}
                \STATE{Send $\left\{\nabla \widehat{F}(\widehat{W}_{t}^{(2l-1)}, D_{k}), \bm{\Psi}_{k}^{(2l-1)}, \bm \Phi_{k}^{(2l-1)}\right\}^{L}_{l=1}$ to the server}
            \ENDFOR
            \STATE { $W_{\mathbf{org\_t+1}}=\{W_{\mathbf{org\_t+1}}^{(l)}\}^{L}_{l=1}\leftarrow\mathbf{MU\_Server}\left(\left\{W_{\mathbf{org\_t}}^{(l)}\right\}^{L}_{l=1}, \left\{\left\{\nabla \widehat{F}(\widehat{W}_{t}^{(2l-1)}, D_{k}), \bm{\Psi}_{k}^{(2l-1)}, \bm \Phi_{k}^{(2l-1)}\right\}^{L}_{l=1}\right\}_{k\in S_{t}}, R_{t}\right)$}
        \ENDFOR
          \RETURN{$W=\left\{W_{\mathbf{org\_T+1}}^{(l)}\right\}^{L}_{l=1}$}
    \end{algorithmic}
    \label{alg:rndphase}
\end{algorithm*}

\subsection{The Overall of Our Scheme}
After solving the above two problems, we can generate the overall process of our scheme, which is illustrated in Fig. \ref{pipeline}. We describe the detailed operations in Algorithm \ref{MP-DP}, where the global model update is similar to Section \ref{subsubsec:recover}. Due to the page limitation, we omit the detailed introduction and directly show the overall pipeline in Algorithm \ref{MU-Server}.

\section{Privacy Analysis}\label{sec:privacy_analysis}
The design goal of our method is to prevent the server from obtaining any information about each client's private training data while preventing clients from obtaining any information related to the global model as well as other clients' training data. Therefore, we analyze privacy from two different perspectives: the privacy of the global model and the privacy of the local gradient.

\subsection{Privacy-preservation of The Global Model}\label{sec:privacy:client}
In this section, we demonstrate that the semi-trusted clients can obtain neither the true global model nor the local gradients. Since each client $k$ can obtain the perturbed global model parameters $\{\widehat{W}^{(l)}\}^{2L-1}_{l=1}$ from the server and compute the perturbed local gradients $\{\nabla F(\widehat{W}^{(l)}, D_{k})\}_{l=1}^{2L-1}$, we need to prove that the client cannot get the true global model $\{W^{(2l-1)}\}^{L}_{l=1}$ and true local gradients $\{\nabla F(W^{(2l-1)}, D_{k})\}_{l=1}^L$ from them. Note that in Theorem \ref{theorem:relation}, we only care about the privacy of model parameters and local gradients in odd-indexed layers.

The privacy of the global model is given in Theorem \ref{theo:para}, whose proof is shown in Appendix \ref{appendix_theorem6}.
\begin{theorem}\label{theo:para}
For any given perturbed global model $\widehat{W}=\{\widehat{W}^{(l)}\}^{2L-1}_{l=1}$, the possibility that clients obtain true global model $\{W^{(2l-1)}\}^{L}_{l=1}$ is equal to zero.
\end{theorem}

Similarly, we conclude the privacy of the local gradients, as shown in Theorem \ref{theo:gradient}. We omit the proof since it is similar to Theorem \ref{theo:para}.
\begin{theorem}\label{theo:gradient}
For any computed perturbed gradients $\{\nabla F(\widehat{W}^{(2l-1)}, D_{k})\}^{L}_{l=1}$, the possibility that clients obtain the true gradients $\{\nabla F(W^{(2l-1)}, D_{k})\}^{L}_{l=1}$ is equal to zero.
\end{theorem}

Based on Theorems \ref{theo:para} and \ref{theo:gradient}, the privacy-preservation of the global model is shown in Theorem \ref{theo:client}.
\begin{theorem}\label{theo:client}
Given the perturbed global model $\{\widehat{W}^{(l)}\}^{2L-1}_{l=1}$, our scheme ensures that each semi-trusted client learns nothing about the true global model $\{W^{2l-1}\}^{L}_{l=1}$ (including true local gradients $\{\nabla F(W^{(2l-1)}, D_{k})\}^{L}_{l=1}$), except for the information deriving from its input $D_{k}$.
\end{theorem}

\subsection{Privacy-preservation of Local Gradients}
In this part, we prove that our scheme can prevent the semi-trusted server from getting the privacy training data of each client. Specifically, with the private noises $\{R^{a}, \{R^{(l)}\}_{l=1}^{2L-1}\}$, the server can recover the local gradient as: for any $1 \leq l \leq L$, 
\begin{flalign*}
    \nabla \widehat{F}(W^{(2l-1)}, D_{k}) =&\nabla F(W^{(2l-1)}, D_{k})+ R^{(2l-1)}\circ\bm{\eta}_{k}^{(2l-1)}\\
     &+ R^{(2l-1)}\circ \sum_{v\in N(k)}\bm\Delta^{(2l-1)}_{k, v},
\end{flalign*}
where $R^{(2l-1)}\circ \bm\Delta^{(2l-1)}_{k, v}\sim \mathcal{N}(0, \sigma_{\Delta}^2)$ and $R^{(2l-1)}\circ \bm{\eta}_{k}^{(2l-1)}\sim \mathcal{N}(0, \sigma_{\eta}^2)$. This result follows the Gopa protocol \cite{abs-2006-07218}, which demonstrates that the privacy of the local gradients satisfies the following theorems.

\begin{theorem}[\cite{abs-2006-07218}]\label{thm:diffprivacy-complete}
Let $\delta \in (0,1)$ and $G=(C, E)$ be the complete graph, if $\varepsilon$, $\delta$, $\sigma_{\eta}$ and $\sigma_{\Delta}$ satisfy Eq. \eqref{eq:gopa_dp} with $\theta = \frac{1}{\sigma_{\eta}K} + \frac{1}{\sigma_{\Delta}K}$,
then our $\mathbf{DDP\_Client}$ algorithm satisfies $(\varepsilon,\delta)$-differential privacy.
\begin{equation}\label{eq:gopa_dp}
    \varepsilon \ge \theta/2 + \theta^{1/2} \textnormal{ and } (\varepsilon-\theta/2)^2 \ge 2\log(2/\delta\sqrt{2\pi})\theta
\end{equation}
\end{theorem}
\begin{theorem}[\cite{abs-2006-07218}]\label{thm:diffprivacy-random}
    Let $\delta \in (0,1)$ and $G=(C, E)$ be the random $n$-out graph where each client randomly chooses $n< K$ neighbors. Let $K\geq 81$ and $n$ meets that $n\geq 4\log(2K/3\delta)$, $n\geq 6\log(K/3)$ and $n\geq \frac{3}{2}+\frac{9}{4}\log(2e/\delta)$. If $\varepsilon$, $\delta$, $\sigma_{\eta}$ and $\sigma_{\Delta}$ satisfy Eq. \eqref{eq:gopa_dp} in Theorem \ref{thm:diffprivacy-complete} with
    \begin{equation*}
        \theta = K^{-1}\sigma_{\eta}^{-2}+\left(\frac{1}{\lfloor(n-1)/3\rfloor-1}+\frac{12+6\log K}{K}\right)\sigma_{\Delta}^{-2},
    \end{equation*}
    then our $\mathbf{DP\_Client}$ satisfies $(\varepsilon,3\delta)$-differential privacy.
\end{theorem}
From Theorems \ref{thm:diffprivacy-complete} and \ref{thm:diffprivacy-random}, we can see that both $\sigma_\eta$ and $\sigma_\Delta$ contribute to the privacy parameters  $(\varepsilon,\delta)$, while only $\sigma_\eta$ affects the aggregated result. Therefore, our scheme ensures LDP for the local gradients while simultaneously providing better utility (see Section \ref{sec:performance} for details).

\section{Performance evaluation}\label{sec:performance}
We evaluate our scheme's performance from three perspectives: 1) \emph{computational efficiency}, 2) \emph{model accuracy}, and 3) \emph{defense ability against privacy attacks}.  

\subsection{Experimental Setup} 
All experiments are conducted on a 20-core Intel Xeon Gold 6148 CPU and a single RTX 3090 GPU. We use the SGD optimizer with a learning rate of 0.1, training models for 100 epochs. The batch size for each client is set to 256. Note that similar to \cite{TruexBASLZZ19, XuBZAL19}, we simulate FL setup instead of conducting real distributed experiments, thereby excluding network latency from the comparisons. To represent different deployment settings, we test with two numbers of clients, $K=5$ and $K=100$, corresponding to small and large scenarios, respectively. We set $n=5$ to indicate that random $5$ of $K$ clients interconnected to negotiate pairwise-correlated noises $\Delta^{(2l-1)}_{k, v}$.

\subsubsection{Baseline Methods} To show the performance advantage, we compare our scheme with SOTAs, i.e., FedAvg \cite{McMahanMRHA17}, PILE \cite{TangSLZXQ23}, TP-SMC \cite{TruexBASLZZ19} and HybridAlpha \cite{XuBZAL19}. FedAvg \cite{McMahanMRHA17} is one of the widely used baselines in FL-related works, which does not consider privacy. PILE \cite{TangSLZXQ23}, TP-SMC \cite{TruexBASLZZ19} and HybridAlpha \cite{XuBZAL19} are SOTAs that consider the privacy of the global model and local gradients. Both \cite{TruexBASLZZ19} and \cite{XuBZAL19} adopt DP to protect the privacy of local gradients and exploit SMC to encrypt DP-perturbed gradients so that only the server can decrypt the aggregated result. However, the server in both schemes directly sends the plaintext global model to each client during the middle iterations. PILE \cite{TangSLZXQ23} is the first to fully implement bidirectional privacy protection by combining model perturbation \cite{0003FFSTXL22} with threshold HE \cite{DamgardJ01}.


\subsubsection{Datasets and Models}
We use two widely recognized datasets in classification tasks: CIFAR10 and CIFAR100. CIFAR10 contains 60,000 32x32 color images across 10 classes (50,000 for training and 10,000 for testing). CIFAR100 is a more challenging dataset, containing the same number of images as CIFAR10, but divided into 100 classes. Regarding the model architecture, we adopt ResNet20 with 20 layers and ResNet56 with 56 layers for the CIFAR10 and CIFAR100, respectively. Each model is implemented in PyTorch, chosen for its balance of computational feasibility and high classification accuracy on its corresponding dataset.


\subsubsection{Privacy Attack Setup} 
We compare the defense capabilities of our scheme with SOTAs against membership inference and reconstruction attacks.
\begin{itemize}
\item \textbf{Membership Inference Attacks}
aims to tell if a certain data record is part of a training set with a surrogate attack model \cite{NasrSH19}. 
The attack model takes the final prediction or intermediate output of target models as input and outputs two classes ``\textit{Member}'' or ``\textit{Non-member}''. We implemented it based on the method proposed in \cite{BalutaSHTS22}. We assume the attacker accesses a fraction of ($50\%$) the training set and some non-member samples. To balance the training, we select half of each batch to include member instances and the other half non-member instances from the attacker's background knowledge, which prevents the attack model from being biased toward member or non-member instances. 

\item \textbf{Reconstruction Attacks} strive to reconstruct original training data utilizing the local gradients uploaded by each client \cite{ZhuLH19}. We implemented it based on the method proposed in \cite{FangCWWX23}. We evenly split the training set based on labels and use one half as the private set and the remainder as the public set. We evaluate the defense ability under the most strict setting, i.e., the attacker cannot own any auxiliary knowledge about private images, where he/she will recover the image from scratch. 
\end{itemize}

\subsection{Computational Efficiency}
We compare the computational costs of our scheme with SOTAs under different numbers of clients, and the corresponding results are given in Tab. \ref{tab:Computational}. 
From the table, we can see that for different numbers of clients, 
the computational efficiency of our scheme is almost the same as the privacy-ignoring FedAvg \cite{McMahanMRHA17} while outperforming PILE \cite{TangSLZXQ23}, TP-SMC \cite{TruexBASLZZ19} and HybridAlpha \cite{XuBZAL19}. 

The main reason for this advantage lies in the different methods used to protect the privacy of the global model. In our scheme, we introduce a novel model perturbation technique as an alternative to the time-consuming cryptographic methods (e.g., threshold HE in TP-SMC \cite{TruexBASLZZ19} and functional encryption in HybridAlpha \cite{XuBZAL19}). Our scheme requires only a few element-wise matrix multiplication and addition operations, while local model training follows the FedAvg framework. Thus, the additional computational costs of these element-wise operations over the real number field are almost negligible. To explain the complexity of our scheme, especially in terms of setup and maintenance of noise mechanisms, we add detailed comparisons in appendix \ref{appendix_Computational_efficiency_details}.

\begin{table*}[h]
    \centering
    \caption{Computational efficiency comparison of SOTAs. For each dataset, we report server, client, and per-epoch computation times, along with the total time to reach target accuracy (ACC: 65\% for CIFAR10 and 40\% for CIFAR100).}
    \begin{tabular}{cccccccccc}
        \toprule
        \multirow{3}{*}{\textbf{Dataset}}&\multirow{3}{*}{\textbf{Method}} & \multicolumn{4}{c}{$K=n=5$}& \multicolumn{4}{c}{$K=100$ and $n=5$}\\
        \cmidrule(lr){3-6}\cmidrule(lr){7-10}
        &&	Server&Client& One Epoch&	ACC  &Server &Client&One Epoch  & ACC   \\
        
        \midrule
        \multirow{5}{*}{\textbf{CIFAR10}}&
        \textbf{FedAVG \cite{McMahanMRHA17}} 
        & 0.05s & 1.32s & 1.37s & 65.68s & 0.07s & 1.40s & 1.47s & 71.31s\\
        &\textbf{PILE \cite{TangSLZXQ23}}     
        & 0.87h & 0.14h & 1.01h & 40.40h & 2.85h & 0.15h & 3.00h & 5.54d \\
        &\textbf{TP-SMC \cite{TruexBASLZZ19}} 
        & 2.25s & 32.12s & 34.37s & 36.69m & 40.88s & 35.77s & 1.28m & 1.37h\\
        &\textbf{HybridAlpha \cite{XuBZAL19}} 
        & 15.16m & 7.35m & 22.51m & 20.29h & 1.29h & 7.46m & 1.42h & 3.89d\\
        &\textbf{Our scheme}   	
        & 0.08s & 2.11s & 2.19s & 2.03m & 0.26s & 2.18s & 2.44s & 2.46m \\
        
        \midrule
        \multirow{5}{*}{\textbf{CIFAR100}}&
        \textbf{FedAVG \cite{McMahanMRHA17}} 
        & 0.25s & 5.36s & 5.61s & 218.44s & 0.54s & 5.38s & 6.12s & 237.06s\\
        &\textbf{PILE \cite{TangSLZXQ23}}     
        & 1.90h & 0.38h & 2.28h & 91.34h & 6.50h & 0.42h & 6.90h & 12.64d\\
        &\textbf{TP-SMC \cite{TruexBASLZZ19}} 
        & 15.50s & 84.24s & 68.74s & 73.38m & 1.36m & 1.36m & 2.72m & 2.73h\\
        &\textbf{HybridAlpha \cite{XuBZAL19}} 
        & 1.01h & 0.38h & 1.39h & 76.34h & 4.31h & 0.39h & 4.72h & 10.90d\\
        &\textbf{Our scheme} 
        & 0.27s & 6.46s & 6.63s & 5.43m & 0.71s & 5.93s & 7.18s & 7.40m\\

        \bottomrule
    \end{tabular}
    \label{tab:Computational}
\end{table*}

\begin{figure*}[htbp]
    \centering
    \caption{Accuracy comparison under different privacy budgets, where $n=5$.}
    \label{fig:accuracy}
    \includegraphics[width=1\textwidth]{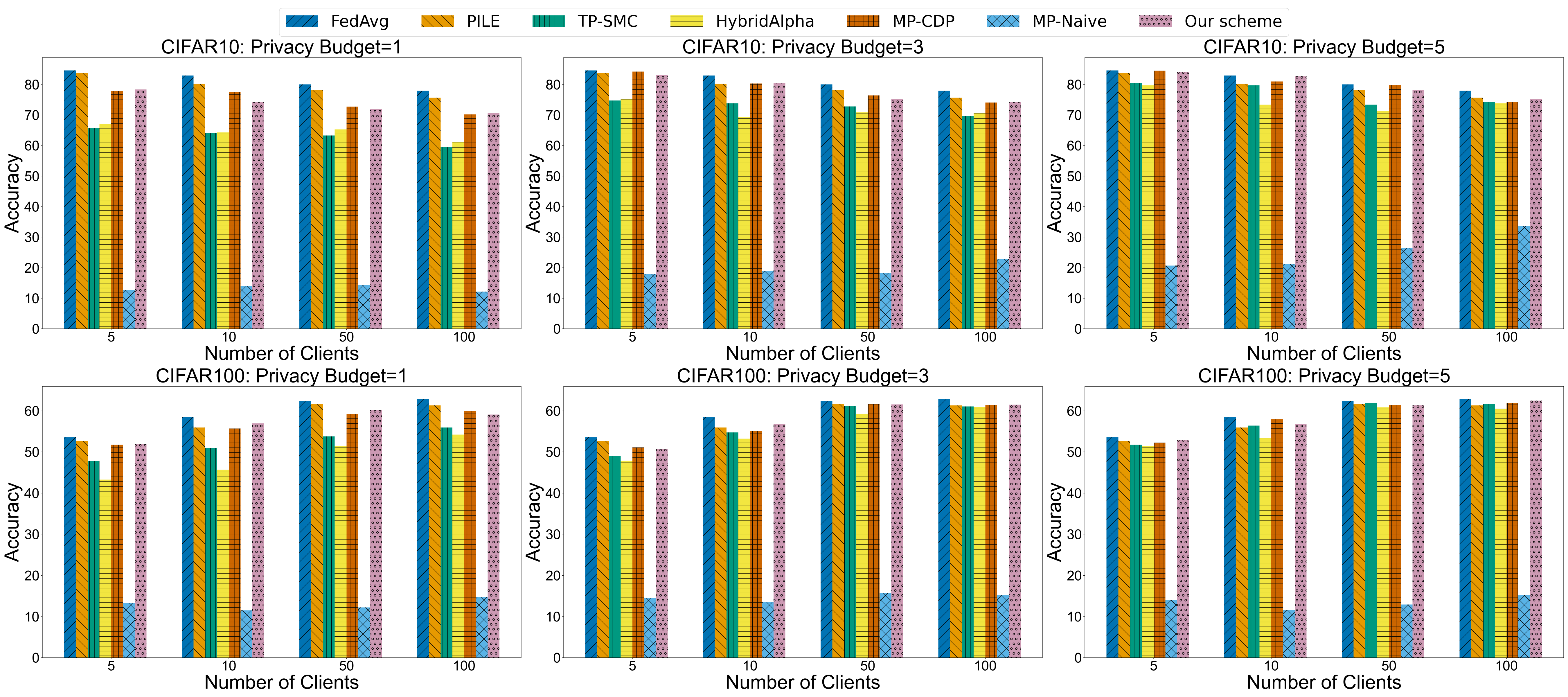}
\end{figure*}


\subsection{Model Accuracy} \label{performance}
We compare the model accuracy of our scheme with SOTAs. Besides, to evaluate specific components of our scheme, we introduce two additional baselines, MP-DP-Naive and MP-CDP. MP-DP-Naive is the scheme that directly concatenates \cite{0003FFSTXL22} and DP (i.e., $\bm\eta_{k}^{(l)}$ in Eq. \eqref{UPDP} satisfies $\mathcal{N} (0, \sigma^2)$ rather than $R^{(l)}\circ \bm\eta_{k}^{(l)}$ in Eq. \ref{eq:recovery}). We conduct the model accuracy of MP-DP-Naive to show the ingenuity of our design. The MP-CDP is the scheme that simulates the model accuracy with CDP, where we directly add the DP-noises to the true aggregated result $\nabla F(W^{(l)})$ rather than local gradients of each client. We simulate the MP-CDP scheme to show that our scheme can indeed effectively address Problem \ref{Prob2}.

Figure \ref{fig:accuracy} gives the comparison result of model accuracy. It shows that as the privacy budget increases, the accuracy loss of schemes decreases except the MP-DP-Naive. Specifically, the accuracy of our scheme is significantly higher than that of TP-SMC and HybridAlpha. Furthermore, for different numbers of clients, the model accuracy of our scheme is almost the same as that of MP-CDP as the privacy budget increases. Besides, the accuracy loss of the MP-DP-Naive is too large to be acceptable. The main reason is that $R^{(l)}$ is the positive real matrix, if each client directly adds DP-noises to the perturbed local gradient, then $R^{(l)}\circ \bm\eta_{k}^{(l)}$ will cause DP-noises expansion, thus drastically reducing model accuracy. 

\begin{table*}[htbp]
\centering
\caption{Comparison of different defense methods against reconstruction attacks, where privacy budget=3.}
\begin{tabular}{ccccccccc}
\toprule
 \multirow{3}{*}{\textbf{Method}}&\multicolumn{4}{c}{\textbf{CIFAR10}}&\multicolumn{4}{c}{\textbf{CIFAR100}}\\
 \cmidrule(lr){2-5} \cmidrule(lr){6-9}
 & \textbf{Accuracy}  & \textbf{PSNR} & \textbf{Test-MSE}&\textbf{Feat-MSE}& \textbf{Accuracy}  & \textbf{PSNR} & \textbf{Test-MSE}&\textbf{Feat-MSE} \\   
\midrule
\textbf{FedAVG \cite{McMahanMRHA17}}	
& $\textbf{84.60\%}$    & $13.56$ & $0.80$ &  $0.15$	&$53.58\%$	&  $14.10$&  $0.46$&  $0.12$\\
\cmidrule(lr){1-9}

\textbf{PILE \cite{TangSLZXQ23}}		
& $83.74\%$             & $13.68$ & $0.78$ &  $0.14$  	&$52.70\%$	&  $13.67$&  $0.54$&  $0.14$\\
\textbf{TP-SMC \cite{TruexBASLZZ19}}	
& $65.66\%$             & $10.39$ & $1.54$ &  $1.29$  	&$48.98\%$	&  $10.70$&  $1.72$&  $1.25$  \\

\textbf{HybridAlpha \cite{XuBZAL19}}	
& $67.13\%$             & $10.21$ & $1.53$ &  $1.91$  	&$47.86\%$	&  $10.54$&  $1.40$&  $1.34$  \\
\textbf{MP-CDP}                       
& $84.21\%$             &$12.15$  & $0.91$ &  $0.51$  	&$51.03\%$	&  $12.85$&  $0.40$&  $0.23$ \\
\textbf{Our scheme}							
& $83.15\%$             &$\textbf{9.22}$ & $\textbf{2.00}$ & $\textbf{4.13}$  	&$50.64\%$	&  $\textbf{8.49}$&  $\textbf{2.57}$&  $\textbf{3.78}$\\ 

\bottomrule
\end{tabular}
\label{tab:reResult}
\end{table*}

\subsection{Defense Ability Against Privacy Attacks}
We evaluate the defense ability against membership inference and reconstruction attacks. 

\subsubsection{Membership Inference Attacks}
We use the global model obtained by clients to identify whether a data record is used during the training phase. For a comprehensive comparison, we report two metrics: 1) \textbf{accuracy}, which is the prediction accuracy on classification tasks for the finally trained model, and 2) \textbf{attack success rate (ASR) against client-side models}.

We present the result in Tab. \ref{tab:miaResult}, which shows that the defense ability of our scheme is much better than that of SOTAs. Particularly, the ASR of our scheme and PILE is $\sim50\%$, which means clients cannot infer any information from the given perturbed model except blind guessing. Note that for membership inference attacks that infer a particular record is either a “member” or a “non-member”, the worst attack success rate is $50\%$ (i.e., blind guessing). However, the attack success rates of SOTAs are much higher than $50\%$. Compared to the $50\%$ baseline, FedAvg shows a $22\%$ advantage in ASR on CIFAR10 and a $17\%$ advantage on CIFAR100, while TP-SMC and HybridAlpha show $\sim 10\% $ advantage in ASR on both datasets. This implies that clients have a distinct advantage in successfully launching an attack compared to blind guessing.
The main reason is that, unlike SOTAs that either exchange plaintext global model during the middle iterations (i.e., TP-SMC and HybridAlpha) or even ignore privacy (i.e., FedAvg), both our scheme and PILE require clients to train on the perturbed global model. In this way, the server's global model is effectively protected from being accessed by clients.

\begin{table*}[htbp]
    \centering
    \caption{Performance Comparison of Privacy-Preserving Methods in Breast Cancer Classification, where privacy budget=3.}
    \begin{tabular}{cccccc}
        \toprule
         &\textbf{Method}& \textbf{Accuracy}  & \textbf{Epoch Time} & \textbf{PSNR}  & \textbf{ASR}\\
        \midrule
        \multirow{5}{*}{$K=n=5$}&
         \textbf{FedAVG \cite{McMahanMRHA17}} 
         & $\textbf{56.95\%}$ & \textbf{25.99}s & 15.61 & $73.05\%$\\
         &\textbf{PILE \cite{TangSLZXQ23}}     
         & $55.47\%$ & 54146.37s & 11.03 & $58.95\%$\\         
         &\textbf{TP-SMC \cite{TruexBASLZZ19}} 
         & $55.15\%$ & 1252.25s & 12.26 & $50.27\%$\\
         &\textbf{HybridAlpha \cite{XuBZAL19}} 
         & $53.84\%$ & 23377.67s & 12.71 & $56.82\%$\\
         &\textbf{Our scheme} 
         & $56.47\%$ & 49.28s & \textbf{9.81} & $\textbf{50.18\%}$\\
         
         \midrule
          \multirow{5}{*}{\makecell{$K=100$ \\ and \\$n=5$}}&
         \textbf{FedAVG \cite{McMahanMRHA17}} 
         & $\textbf{60.31\%}$ & \textbf{26.01s} & 15.30 & $72.36\%$\\
         &\textbf{PILE \cite{TangSLZXQ23}}     
         & $59.92\%$ & 170174.57s & 12.13 & $55.11\%$\\         
         &\textbf{TP-SMC \cite{TruexBASLZZ19}} 
         & $58.49\%$ & 1533.95s & 12.38 & $50.45\%$\\
         &\textbf{HybridAlpha \cite{XuBZAL19}} 
         & $57.09\%$ & 82898.49s & 12.50 & $55.68\%$\\
         &\textbf{Our scheme} 
         & $58.98\%$ & 51.62s & \textbf{9.79} & $\textbf{50.24}\%$\\
        \bottomrule
    \end{tabular}
    \label{tab:medical}
\end{table*}

\begin{table}[h]
    \centering
    \caption{Comparison of different defense methods against membership inference attacks, where privacy budget=3.}
    \begin{tabular}{ccccc}
        \toprule
           \multirow{3}{*}{\textbf{Method}}& \multicolumn{2}{c}{\textbf{CIFAR10}}  & \multicolumn{2}{c}{\textbf{CIFAR100}} \\
         \cmidrule(lr){2-3} \cmidrule(lr){4-5}
         & \textbf{Accuracy}  & \textbf{ASR} & \textbf{Accuracy}  & \textbf{ASR}\\
        \midrule
        
         \textbf{FedAVG \cite{McMahanMRHA17}} 
         & $84.60\%$  & $72.01\%$ & $53.58\%$	& $67.30\%$\\
         \textbf{PILE \cite{TangSLZXQ23}}     
         & $83.74\%$   & $50.44\%$ & $52.70\%$	& $50.85\%$\\         
         \textbf{TP-SMC \cite{TruexBASLZZ19}} 
         & $65.66\%$   & $58.19\%$ & $48.98\%$	& $57.18\%$\\
         \textbf{HybridAlpha \cite{XuBZAL19}} 
         & $67.13\%$   & $59.49\%$& $47.86\%$	& $60.77\%$\\
         \textbf{Our scheme} 
         & $83.15\%$  & $\mathbf{50.30\%}$  & $50.64\%$	& $\mathbf{50.23\%}$\\
        \hline
    \end{tabular}
    \label{tab:miaResult}
\end{table}

\subsubsection{Reconstruction Attacks}
We launch reconstruction attacks to recover the original training data from local gradients. For a comprehensive comparison, in addition to model accuracy, we report the following three metrics:

\begin{itemize}  
    \item Peak Signal-to-Noise Ratio (PSNR) is a widely utilized metric in image processing for evaluating the quality of reconstructed images. A higher PSNR indicates better quality and higher similarity to the original image.
    \item Test Mean Squared Error (Test-MSE) quantifies the average squared difference between reconstructed images and ground truth images. The lower the Test-MSE, the higher the similarity between reconstructed images and original images.
    \item Feature Mean Squared Error (Feat-MSE) calculates the mean squared error between high-level feature representations of reconstructed and original images. Lower Feat-MSE suggests a closer match in the underlying image features.
\end{itemize}

Tab. \ref{tab:reResult} gives the comparison results of reconstruction attacks launched by the server. From the table, we can see that our scheme outperforms the privacy-ignoring FedAVg and other SOTAs (i.e., TP-SMC, HybridAlpha, PILE, and MP-CDP) in defense against reconstruction attacks. It is worth noting that the model accuracy of our scheme is consistent with that of MP-CDP, but the defense ability is better. This is because of the introduction of pairwise-correlated noises, which can eventually be canceled. These pairwise-correlated noises compensate for the independent noises well, making the noises added by the local gradient larger but the aggregated noises relatively small. We also provide the visualization of reconstructed images in the appendix \ref{appendix_Reconstruct}.

%

\subsection{Performance in The Real-world Scenario}
In this section, we evaluate the performance of our scheme in a realistic medical scenario, specifically for a breast cancer classification task \cite{7312934}. We apply a ResNet11 model to identify breast cancer from patient images and diagnostic data. The evaluation metrics include accuracy (to assess diagnostic reliability), epoch time (to measure computational efficiency), PSNR (to evaluate resilience against reconstruction attacks), and ASR (to gauge privacy protection against membership inference attacks)

Tab \ref{tab:medical} demonstrates that compared with SOTAs, our scheme achieves an optimal balance among computational efficiency, model accuracy, and robustness against privacy attacks. Whether in terms of PSNR, ASR values or accuracy, our scheme has the strongest privacy protection capability, and its accuracy is close to FedAvg, thus maintaining the reliability of diagnosis. Additionally, the computational efficiency significantly outperforms than that of SOTAs. Consequently, our scheme presents a practical solution for healthcare applications that require both high diagnostic accuracy and robust privacy protection.

\section{Conclusion}\label{sec:conclusion}
In this paper, we present an efficient and bidirectional privacy-preserving scheme for FL, which can ensure the privacy of both local gradients and the global model while maintaining model accuracy. Specifically, we skillfully design the constraint on the random one-time-used noises for the model perturbation method to prevent semi-trusted clients from obtaining the true global model. Meanwhile, we elegantly present the novel differential privacy mechanism on the client side to achieve the privacy of local gradients while maintaining the same utility level as the CDP. Extensive experiments demonstrate the computational efficiency of our scheme is almost comparable to that of the privacy-ignoring scheme and significantly outperforms SOTAs. Besides, empirical results of model accuracy show that our scheme can achieve the same utility as the CDP and is much better than SOTAs. Particularly, when the privacy budget $\epsilon$ is set relatively small, the accuracy loss of our scheme is less than $6\%$ compared to FedAvg, while the other SOTAs suffer up to $20\%$ accuracy loss. Furthermore, experimental results on defending against privacy attacks show that the defense capability of our scheme also outperforms than SOTAs.

\bibliographystyle{IEEEtran}
\bibliography{example_paper}

\appendices
\section{Proofs}
In this section, we give the proofs of Theorems \ref{defactorization}-\ref{theo:para}.
\subsection{Proof of Theorem \ref{defactorization}}\label{appendix_theorem3}
\begin{proof}
Let $f(t)$ be the characteristic function of the random variable $\ln |Z|$,  we have
$$
\begin{aligned}
f(t)&=\mathrm{E}( e^{i t \ln |Z|})=\int_{-\infty}^{\infty} e^{i t \ln |z|} \frac{1}{\sigma\sqrt{2 \pi}} e^{-\frac{z^2}{2\sigma^2}} \mathrm{~d} z \\
& =\frac{2}{\sigma\sqrt{2 \pi}} \int_0^{\infty} z^{i t} e^{-\frac{z^2}{2\sigma^2}} \mathrm{~d} z =\frac{(\sqrt{2}\sigma)^{i t}}{\sqrt{\pi}} \int_0^{\infty} \frac{u^{\frac{it-1}{2}}}{e^{u}}  \mathrm{~d} u \\
& =(\sqrt{2}\sigma)^{i t}\Gamma\left(\frac{1+i t}{2}\right) / \Gamma\left(\frac{1}{2}\right) \\
&  \stackrel{(a)}{=}\exp \left\{i t \ln(\sqrt{2}\sigma)\right\} \frac{1}{1+i t} \prod_{\ell=1}^{\infty} \frac{\exp \left\{\frac{i t}{2} \ln \left(1+\frac{1}{\ell}\right)\right\}}{1+\frac{i t}{2 \ell+1}},
\end{aligned}
$$
where Eq. $(a)$ follows from the famous Euler's product formula $\Gamma(z)=\frac{1}{z} \prod_{\ell=1}^{\infty} \frac{\left(1+\frac{1}{\ell}\right)^z}{1+\frac{z}{\ell}}$. Then by the property of characteristic function, we have
\[\begin{aligned}
    \ln |Z| &= -\sum_{\ell=1}^{\infty}\left[\frac{E_\ell}{2 \ell+1}-\frac{1}{2} \ln \left(1+\frac{1}{\ell}\right)\right]+\ln(\sqrt{2}\sigma)-E_0\\
    &=-(\sum_{\ell=1}^{\infty}\left[\frac{G_{1/m, j}}{2 \ell+1}-\frac{1}{2m} \ln \left(1+\frac{1}{\ell}\right)\right]+\\
    & \sum_{\ell=1}^{\infty}\left[\frac{G_{1/m, j}}{2 \ell+1}-\frac{1}{2m} \ln \left(1+\frac{1}{\ell}\right)\right]+\cdots+\\
    &\sum_{\ell=1}^{\infty}\left[\frac{G_{1/m, j}}{2 \ell+1}-\frac{1}{2m} \ln \left(1+\frac{1}{\ell}\right)\right])+\frac{\ln(\sqrt{2}\sigma)-E_0}{n}\\
    &~~~~+\frac{\ln(\sqrt{2}\sigma)-E_0}{n}+\cdots+\frac{\ln(\sqrt{2}\sigma)-E_0}{n}
\end{aligned}\]
where $E_0, E_1, \ldots$ are i.i.d. standard exponential random variables.
\end{proof}

\subsection{Proof of Theorem \ref{theorem:relation}}\label{appendix_theorem4}
We present some properties of $\{R^{(l)}\}_{l=1}^{2L-1}$ given in Eq. \eqref{eq:Rrevised1}:
\begin{eqnarray*}
    \left\{
        \begin{aligned}
            & R^{(1)}= D_{\bm r^{(1)}}E^{(1)} \textnormal{ and } R^{(2L-1)} D_{\frac{1}{\bm s^{(L-1)}}}=E^{(L)},\\
            &\begin{aligned}
                & R^{(2l)}D_{\bm r^{(l)}}=  D_{\frac{1}{\bm s^{(l)}}}E^{(l)}, \\
            & R^{(2l-1)}D_{ \frac{1}{\bm s^{(l-1)}} }=D_{\bm r^{(l)}}E^{(l)},\\
            \end{aligned} \textnormal{ for } 2\leq l\leq L-1,
        \end{aligned}
    \right.
\end{eqnarray*}
 where $D_{\bm r^{(l)}}$ (resp. $D_{\frac{1}{\bm s^{(l)}}}$) is the $n_{l} \times n_{l}$ diagonal matrix whose main diagonal is $\bm r^{(l)}$ (resp. $\frac{1}{\bm s^{(l)}}$) and $E^{(l)}$ is the $n_{l} \times n_{l}$ matrix whose entries are all 1s.
For $1 \leq l \leq L-1$, denote $\bm{\hat{y}}^{(2l-1)}$, $\bm{\hat{y}}^{(2l)}$ as the perturbed output vectors,  and $\bm y^{(2l-1)}$, $\bm y^{(2l)}$ as the true output vector. 
\begin{theorem}\label{theorem:output}
Let $\alpha=\sum_{i=1}^{n_{L-1}} \hat {y}^{(2L - 2)}_i$ and $\bm{r}=\bm \gamma \circ \bm{r}^{(a)}$, then for $1 \leq l \leq L-1$, $\bm{\hat{y}}^{(2l-1)}=  \bm r^{(l)} \circ \bm y^{(2l-1)}$,~$\bm{\hat{y}}^{(2l)}=  \frac{1}{\bm s^{(l)}} \circ \bm y^{(2l)}$, and $\bm{\hat{y}}^{(2L-1)} = \bm y^{(2L-1)} +  \alpha \bm{r}$.
\end{theorem}

\emph{Proof.}  We prove the theorem by mathematical induction. Firstly, when $l=1$, we can obtain
$$\begin{aligned}
    \hat {\bm y}^{(1)} &= ReLU\left(\widehat W^{(1)} \bm x\right)=  ReLU\left(\big(R^{(1)} \circ W^{(1)}\big) \bm x\right) \\
    &=ReLU\left(\big(D_{\bm r^{(1)}}E^{(1)} \circ W^{(1)}\big) \bm x\right)\\
    &=ReLU\left(D_{\bm r^{(1)}}\big(E^{(1)} \circ W^{(1)}\big) \bm x\right)\\
    &=  ReLU\left(D_{\bm r^{(1)}} W^{(1)} \bm x\right)=  ReLU\left(\bm r^{(1)}\circ (W^{(1)} \bm x)\right)\\
    & =\bm r^{(1)}\circ ReLU( W^{(1)} \bm x)
    =\bm r^{(1)} \circ \bm y^{(1)}.
\end{aligned}$$
\begin{eqnarray*}
    \hat {\bm y}^{(2)} &=& ReLU\left(\big(R^{(2)} \circ W^{(2)}\big) (\bm r^{(1)} \circ \bm y^{(1)})\right) \\
    &=&  ReLU\left(\big(R^{(2)}D_{\bm r^{(1)}} \circ W^{(2)}\big) \bm y^{(1)}\right) \\
    &=&ReLU\left(\big(D_{\frac{1}{\bm s^{(1)}}}E^{(1)} \circ W^{(2)}\big) \bm y^{(1)}\right)\\
    &= & ReLU\left(D_{\frac{1}{\bm s^{(1)}}} W^{(1)} \bm y^{(1)}\right)\\
    & =&\frac{1}{\bm s^{(1)}}\circ ReLU( W^{(2)} \bm y^{(1)})=\frac{1}{\bm s^{(1)}} \circ \bm y^{(2)}.
\end{eqnarray*}
For $2 \leq l \leq L-1$, assuming $\hat {\bm y}^{(2l - 3)}= \bm r^{(l-1)} \circ \bm y^{(2l - 3)}$ and $\bm{\hat{y}}^{(2l-2)}=  \frac{1}{\bm s^{(l-1)}} \circ \bm y^{(2l-2)}$ by induction. Then, we have

\begin{eqnarray*}
     \hat {\bm y}^{(2l-1)} &=& ReLU\left(\widehat W^{(2l-1)} \hat {\bm y}^{(2l - 2)}\right)\\
    &=&ReLU\left(\big(R^{(2l-1)} \circ W^{(2l-1)}\big) D_{ \frac{1}{\bm s^{(l-1)}} }  \bm y^{(2l - 2)}\right)\\
    &=&ReLU\left(\big( (R^{(2l-1)}D_{ \frac{1}{\bm s^{(l-1)}} }) \circ W^{(2l-1)}\big) \bm y^{(2l - 2)}\right) \\
    &=&ReLU\left(\big( D_{\bm r^{(l )}}E^{(l)} \circ W^{(2l-1)} \big) \bm y^{(2l - 2)}\right)\\
   &=&ReLU\left( D_{\bm r^{(l)} }W^{(2l-1)} \bm y^{(2l - 2)}\right)\\
   &=& ReLU\left(\bm r^{(l)} \circ (W^{(2l-1)} \bm y^{(2l - 2)})\right)\\
   & =& \bm r^{(l)} \circ ReLU\left(W^{(2l-1)} \bm y^{(2l - 2)}\right)=\bm r^{(l)} \circ \bm y^{(2l-1)}.
\end{eqnarray*}

\begin{align*}
    \hat {\bm y}^{(2l)} &=ReLU\left(\big(R^{(2l)} \circ W^{(2l)}\big) D_{ \bm r^{(l)}}  \bm y^{(2l - 1)}\right)\\
    &=ReLU\left(\big( D_{\frac{1}{\bm s^{(l)}}}E^{(l)} \circ W^{(2l)} \big) \bm y^{(2l - 1)}\right)\\
   &= ReLU\left(\frac{1}{\bm s^{(l)}}\circ (W^{(2l)} \bm y^{(2l - 1)})\right)\\
   & = \frac{1}{\bm s^{(l)}} \circ ReLU\left(W^{(2l)} \bm y^{(2l - 1)}\right)=\frac{1}{\bm s^{(l)}}\circ \bm y^{(2l)}.
\end{align*}

Finally,  we prove the last equality as follows.
\begin{eqnarray*}
    &&\hat {\bm y}^{(2L-1)}=\widehat W^{(2L-1)} \hat {\bm y}^{(2L - 2)}  \\
&=& (R^{(2L-1)} \circ W^{(2L-1)} + R^{(a)}) \hat {\bm y}^{(2L-2)} \\
   &=& (R^{(2L-1)} \circ W^{(2L-1)})(\frac{1}{\bm s^{(L-1)}}\circ \bm y^{(2L - 2)}) + R^{(a)}\hat {\bm y}^{(2L-2)}  \\
&=& (R^{(2L-1)} D_{\frac{1}{\bm s^{(L-1)}}}) \circ W^{(2L-1)} \bm y^{(2L - 2)} + R^{(a)}\hat {\bm y}^{(2L-2)}\\
 &=&  W^{(2L-1)} \bm y^{(2L - 2)} + R^{(a)}\hat {\bm y}^{(2L-2)}= \bm y^{(2L-1)} +  \alpha \bm{r}.
\end{eqnarray*}

 \begin{theorem}\label{gradient}
 For each training sample $(\bm{x}, \bm{\bar{y}}) \in \mathcal{D}_{k}$ with label $\bar{\bm y}$ and $1 \leq l \leq 2L-1$, the perturbed gradients $\frac{\partial \mathcal{\widehat L}\left(\widehat{W}; (\bm{x},\bm{\bar{y}})\right)}{\partial \widehat{W}^{(l)}} \in \mathbb{R}^{n_{l}\times n_{l-1}}$ and the true gradients  $\frac{\partial \mathcal{L}\left(W; (\bm{x},\bm{\bar{y}})\right)}{\partial W^{(l)}} \in \mathbb{R}^{n_{l}\times n_{l-1}}$ satisfy
 \begin{equation*}
\frac{\partial \mathcal{\widehat  L}(\widehat{W}; (\bm{x},\bm{\bar{y}}))}{\partial \widehat {W}^{(l)}}=\frac{1}{R^{(l)}} \circ \frac{\partial \mathcal{L}\left(W; (\bm{x},\bm{\bar{y}})\right)}{\partial {W}^{(l)}} +
  \bm{r}^{\top} \bm \sigma^{(l)} - \upsilon \bm \beta^{(l)},
\end{equation*}
 where  $\bm{r}=\bm \gamma \circ \bm{r}^{(a)}$, ~$\upsilon=\bm{r}^{\top}\bm{r}$,~
 $\bm \sigma^{(l)} = \alpha\frac{\partial \hat {\bm y}^{(2L-1)}}{\partial \widehat {W}^{(l)}}+ \Big(\frac{\partial \mathcal{\hat L}(\widehat{W} (\bm{x},\bm{\bar{y}}))}{\partial \hat{\bm y}^{(2L-1)}}\Big)^{\top}\frac{\partial  \alpha}{\partial \widehat {W}^{(l)}}$, and 
 $\bm \beta^{(l)} = \alpha \frac{\partial \alpha}{\partial \widehat {W}^{(l)}}$.

 \end{theorem}

 \begin{proof}
For $\bm x$ is a $u \times v$ matrix, we regard $\bm x$ as a vector of $\mathbb{R}^{uv}$.
According to Theorem \ref{theorem:output}, we have $\bm{\hat{y}}^{(2L-1)} =\bm y^{(2L-1)} +\alpha \bm{r}$, and thus the perturbed loss function is 
\begin{eqnarray*}
   \mathcal{\widehat{L}}\left(\widehat{W}; (\bm{x},\bm{\bar{y}})\right)&=&\frac{1}{2}\parallel\bm \hat{y}^{(2L-1)} - \bar{\bm y}\parallel^{2}_{2}\\
   &=&\frac{1}{2}\parallel\bm y^{(2L-1)} +\alpha \bm{r}- \bar{\bm y}\parallel^{2}_{2}
\end{eqnarray*}
Then,
\begin{small}
\[\frac{\partial \mathcal{\widehat L}\left(\widehat{W}; (\bm{x},\bm{\bar{y}})\right)}{\partial \hat {\bm y}^{(2L-1)}}=\left(\hat {\bm y}^{(2L-1)}-\bar{\bm y}\right)^{\top}=\frac{\partial \mathcal{ L}\left(W; (\bm{x},\bm{\bar{y}})\right)}{\partial {\bm y}^{(2L-1)}}+ \alpha \bm r^{\top},\]
\end{small}
By the chain rule, we can derive that
\begin{eqnarray*}
   && \frac{\partial \mathcal{\widehat  L}(\widehat{W}; (\bm{x},\bm{\bar{y}}))}{\partial \widehat {W}^{(l)}}=\frac{\partial \mathcal{\widehat L}(\widehat{W}; (\bm{x},\bm{\bar{y}}))}{\partial \hat {\bm y}^{(2L-1)}}
  \frac{\partial \hat {\bm y}^{(2L-1)}}{\partial  \widehat {W}^{(l)}}\\
  &=&\frac{\partial \mathcal{L}(W; (\bm{x},\bm{\bar{y}}))}{\partial {\bm y}^{(2L-1)}} \frac{\partial \hat {\bm y}^{(2L-1)}}{\partial  \widehat {W}^{(l)}}+  \alpha \bm r^{\top}
  \frac{\partial \hat {\bm y}^{(2L-1)}}{\partial  \widehat {W}^{(l)}} \\
  &=&\frac{\partial \mathcal{L}(W; (\bm{x},\bm{\bar{y}}))}{\partial {\bm y}^{(2L-1)}}\frac{\partial  ({\bm y}^{(2L-1)}+\alpha \bm{r}) }{\partial  \widehat {W}^{(l)}}+  \alpha \bm r^{\top}
  \frac{\partial \hat {\bm y}^{(2L-1)}}{\partial  \widehat {W}^{(l)}} \\
  &=&\frac{\partial \mathcal{L}(W; (\bm{x},\bm{\bar{y}}))}{\partial \widehat {W}^{(l)}} +\alpha \bm r^{\top}
  \frac{\partial \hat {\bm y}^{(2L-1)}}{\partial  \widehat {W}^{(l)}}+\\
  &&\left(\frac{\partial \mathcal{\widehat L}(\widehat{W}; (\bm{x},\bm{\bar{y}}))}{\partial \hat{{\bm y}}^{(2L-1)}} -  \alpha \bm r^{\top}\right)\frac{\partial (\alpha \bm r )}{\partial \widehat {W}^{(l)}}\\
  &=&\frac{\partial \mathcal{L}(W; (\bm{x},\bm{\bar{y}}))}{\partial \widehat {W}^{(l)}} 
  - \upsilon \alpha \frac{\partial  \alpha}{\partial \widehat {W}^{(l)}} +\\
  &&\bm r^{\top} \left( \alpha\frac{\partial \hat {\bm y}^{(2L-1)}}{\partial \widehat {W}^{(l)}}
  + \Big(\frac{\partial \mathcal{\widehat L}(\widehat{W}; (\bm{x},\bm{\bar{y}}))}{\partial \hat{\bm y}^{(2L-1)}}\Big)^{\top}\frac{\partial  \alpha}{\partial \widehat {W}^{(l)}}\right)\\
  &=& \frac{1}{R^{(l)}} \circ \frac{\partial \mathcal{L}(W; (\bm{x},\bm{\bar{y}}))}{\partial {W}^{(l)}} +
  \bm{r}^{\top}  \bm \sigma^{(l)} - \upsilon \bm \beta^{(l)}.
\end{eqnarray*}
 \end{proof}

\emph{Proof of Theorem \ref{theorem:relation}}:  Suppose $\{W_{\mathbf{org}}^{(l)}\}_{l=1}^L$ is the original global model  and $W=\{W^{(l)}\}_{l=1}^{2L-1}$ is the expanded global model as given in Eq. \eqref{eq:revisedW}. Denote the original true output vector as $\bm{y}_{\mathbf{org}}^{(l)}$. Then  by mathematical induction, for $1 \leq l \leq L-1$, we have
\begin{eqnarray*}
    \bm{y}_{\mathbf{org}}^{(l)}&=&ReLU(W_{\mathbf{org}}^{(l)}\bm{y}_{\mathbf{org}}^{(l-1)})=ReLU(W^{(2l-1)}\bm{y}^{(2l-3)})\\
    &=&ReLU(W^{(2l-1)}\bm{y}^{(2l-2)})=\bm{y}^{(2l-1)},
\end{eqnarray*}
\begin{eqnarray*}
    \bm{y}_{\mathbf{org}}^{(L)}&=&W_{\mathbf{org}}^{(L)}\bm{y}_{\mathbf{org}}^{(L-1)}=W^{(2L-1)}\bm{y}^{(2L-3)}\\
    &=&W^{(2L-1)}\bm{y}^{(2L-2)}=\bm{y}^{(2L-1)},
\end{eqnarray*}
Here we follow the fact that for $2 \leq l \leq L$,
\[\bm y^{(2l-2)}=ReLU(W^{(2l-2)}\bm y^{(2l-3)})=\bm y^{(2l-3)}.\]
Then from the above, we have
\[\mathcal{L}\left(W; (\bm{x},\bm{\bar{y}})\right)=\frac{1}{2}\parallel\bm y_{\mathbf{org}}^{(L)}- \bar{\bm y}\parallel^{2}_{2}=\mathcal{L}\left(W_{\mathbf{org}}; (\bm{x},\bm{\bar{y}})\right),\]
Hence 
\[\frac{\partial \mathcal{L}\left(W_{\mathbf{org}}; (\bm{x},\bm{\bar{y}})\right)}{\partial {W}_{\mathbf{org}}^{(l)}}=\frac{\partial \mathcal{L}\left(W; (\bm{x},\bm{\bar{y}})\right)}{\partial {W}^{(2l-1)}}.\]
Combining with Theorem \ref{gradient}, we obtain
\begin{eqnarray*}
    \begin{aligned}
        &\frac{\partial \mathcal{L}\left(W_{\mathbf{org}}; (\bm{x},\bm{\bar{y}})\right)}{\partial {W}_{\mathbf{org}}^{(l)}}= \frac{\partial \mathcal{L}\left(W; (\bm{x},\bm{\bar{y}})\right)}{\partial {W}^{(2l-1)}}\\
        & =R^{(2l-1)}\circ\Big(\frac{\partial\mathcal{\widehat  L}(\widehat{W};(\bm{x},\bm{\bar{y}}))}{\partial \widehat {W}^{(2l-1)}}-\bm{r}^{\top} \bm \sigma^{(2l-1)} + \upsilon \bm \beta^{(2l-1)}\Big).
    \end{aligned}
\end{eqnarray*}
By the definitions of $\nabla F( W_{\mathbf{org}}^{(l)}, D_{k})$, $\nabla F(\widehat{W}^{(2l-1)}, \mathcal{D}_{k})$, $\bm{\Psi}_{k}^{(2l-1)}$, and $\bm \Phi_{k}^{(2l-1)}$,
we have 
\begin{eqnarray*}
 &&\nabla F( W_{\mathbf{org}}^{(l)}, D_{k})\\
   &=&  \frac{1}{|\mathcal{D}_{k}|}\sum_{(\bm{x}_{i}, \bm{\bar{y}}_{i})\in \mathcal{D}_{k}}R^{(2l-1)}\circ\Big(\frac{\partial\mathcal{\widehat  L}(\widehat{W};(\bm{x}_i,\bm{\bar{y}}_i))}{\partial \widehat {W}^{(2l-1)}}-\\
    &&\bm{r}^{\top} \bm \sigma_{(\bm{x}_{i}, \bm{\bar{y}}_{i})}^{(2l-1)} + \upsilon \bm \beta_{(\bm{x}_{i}, \bm{\bar{y}}_{i})}^{(2l-1)}\Big)\\
    &=&R^{(2l-1)}\circ\Big(\frac{1}{|\mathcal{D}_{k}|}\sum_{(\bm{x}_{i}, \bm{\bar{y}}_{i})\in \mathcal{D}_{k}}\big(\frac{\partial\mathcal{\widehat  L}(\widehat{W};(\bm{x}_i,\bm{\bar{y}}_i))}{\partial \widehat {W}^{(2l-1)}}-\\
    &&\bm{r}^{\top} \bm \sigma_{(\bm{x}_{i}, \bm{\bar{y}}_{i})}^{(2l-1)} + \upsilon \bm \beta_{(\bm{x}_{i}, \bm{\bar{y}}_{i})}^{(2l-1)}\big)\Big)\\
   &=&R^{(2l-1)} \circ \Big(\nabla F(\widehat{W}^{(2l-1)}, D_{k})-\bm{r}^{\top} \bm \Psi_{k}^{(2l-1)}+\upsilon \bm \Phi_{k}^{(2l-1)}\Big).
\end{eqnarray*}

\subsection{Proofs of Theorems \ref{distribution} and \ref{distribution2}}\label{appendix_theorem45}
\begin{proof}
According to Eqs. \eqref{constraint:r}, \eqref{eq:Rrevised1} and \eqref{constraint:epsilon}, for any $1\leq i \leq n_{l}$ and  $1\leq j \leq n_{l-1}$, we can prove Theorems \ref{distribution} and \ref{distribution2} by dividing $l$ into the following three cases.

\noindent 1) When $l = 1$, we can obtain
\[R^{(1)}_{ij} \cdot \bm\eta_{kij}^{(1)}= \bm r_i^{(1)}\cdot \bm\eta_{kij}^{(1)} \textnormal{ and } R^{(1)}_{ij} \cdot \Delta_{k, v}^{(1)}[ij]=\bm r_i^{(1)}\cdot \Delta_{k, v}^{(1)}[ij].\]
Since $\bm r_i^{(1)} \sim \mathcal{DN}^{*}(1)$,  $\bm\eta_{kij}^{(1)}\sim \mathcal{DN}(\sigma_{\eta}, 1)$ and $\Delta_{k, v}^{(1)}[ij] \sim \mathcal{DN}(\sigma_{\Delta}, 1)$, from Theorem \ref{defactorization}, $R^{(1)}_{ij} \cdot \bm\eta_{kij}^{(1)} \sim \mathcal{N}(0, \sigma_{\eta}^{2})$ and $R^{(1)}_{ij} \cdot \Delta_{k, v}^{(1)}[ij]\sim \mathcal{N}(0, \sigma_{\Delta}^{2})$.

\noindent 2) When $1 < l < L$, we can get 
 \begin{equation*}
      \left\{
      \begin{aligned}
         & R^{(2l-1)}_{ij} \cdot \bm\eta_{kij}^{(2l-1)}= \bm r_i^{(l)}\bm{s}^{(l-1)}_{j}\bm\eta_{kij}^{(2l-1)}\\
     & R^{(2l-1)}_{ij} \cdot \Delta_{k, v}^{(2l-1)}[ij]=\bm r_i^{(l)}\bm{s}^{(l-1)}_{j}\Delta_{k, v}^{(2l-1)}[ij]
      \end{aligned} \right.
  \end{equation*}
  Since $\bm r_i^{(l)}, \bm{s}^{(l-1)}_{j} \sim \mathcal{DN}^{*}(2)$, $\bm\eta_{kij}^{(2l-1)} \sim \mathcal{DN}(\sigma_{\eta}, 1)$ and $\Delta_{k, v}^{(2l-1)}[ij] \sim \mathcal{DN}(\sigma_{\Delta}, 1)$, from Theorem \ref{defactorization}, $R^{(2l-1)}_{ij} \cdot \bm\eta_{kij}^{(2l-1)} \sim \mathcal{N}(0, \sigma_{\eta}^{2})$ and $R^{(2l-1)}_{ij} \cdot \Delta_{k, v}^{(2l-1)}[ij]\sim \mathcal{N}(0, \sigma_{\Delta}^{2})$.

\noindent 3)  When $l = L$, 
  \begin{equation*}
      \left\{
      \begin{aligned}
         & R^{(2L-1)}_{ij} \cdot \bm\eta_{kij}^{(2L-1)}  = \bm{s}_{j}^{(L-1)}\bm\eta_{kij}^{(2L-1)}\\
     &R^{(2L-1)}_{ij} \cdot \Delta_{k, v}^{(2L-1)}[ij]  = \bm{s}_{j}^{(L-1)}\Delta_{k, v}^{(2L-1)}[ij]
      \end{aligned} \right.
  \end{equation*}
  Since $\bm{s}_{j}^{(L-1)}\sim \mathcal{DN}^{*}(1)$, $\bm\eta_{kij}^{(2L-1)} \sim \mathcal{DN}(\sigma_{\eta}, 1)$ and $\Delta_{k, v}^{(2L-1)}[ij] \sim \mathcal{DN}(\sigma_{\Delta}, 1)$, according to Theorem \ref{defactorization}, $R^{(2L-1)}_{ij} \cdot \bm\eta_{kij}^{(2L-1)} \sim \mathcal{N}(0, \sigma_{\eta}^{2})$ and $R^{(2L-1)}_{ij} \cdot \Delta_{k, v}^{(2L-1)}[ij] \sim \mathcal{N}(0, \sigma_{\Delta}^{2})$.

Thus, we deduce that $R^{(2l-1)}\circ \bm\eta_{kij}^{(2L-1)}\sim \mathcal{N}(0, \sigma_{\eta}^{2})$ and $R^{(2l-1)}\circ \Delta_{k, v}^{(2l-1)} \sim \mathcal{N}(0, \sigma_{\Delta}^{2})$ for $1\leq l \leq L$.
\end{proof}

\begin{table*}[h]
        \caption{\small Time cost analysis of various methods on datasets. The table compares different methods with varying numbers of participants ($K=5, 100$) for both server-side and client-side operations. All times are in seconds. The symbol ``\textbf{--}'' indicates that the corresponding operation is not applicable for that method on the specified side. The \textbf{Total Time} column represents the time required to complete one epoch of training. }
        \label{tab:time_cost_detail}
	\centering
	\begin{tabular}{cccccccccccc}
	\toprule
			\multirow{3}{*}{\textbf{Dataset}}&&	\multirow{3}{*}{\textbf{Method}}	&
			\multicolumn{5}{c}{\textbf{Server}} &
			\multicolumn{3}{c}{\textbf{Client}} &
			\multirow{3}{*}{\textbf{Total Time }}			\\
			\cmidrule(lr){4-8} \cmidrule(lr){9-11}
		&&	& MP	& Recover & Agg		& Grad\_Dec	& Key\_Gen		& Train	& Grad\_Enc	& DP	&	\\
    \midrule
\multirow{10}{*}{\textbf{CIFAR10}} &
\multirow{5}{*}{\textbf{K=5}}
& \textbf{FedAVG \cite{McMahanMRHA17}} & \textbf{--} & \textbf{--} & $0.05$ & \textbf{--} & \textbf{--} & $1.32$ & \textbf{--} & \textbf{--} & $1.37$ \\
&& \textbf{TP-SMC \cite{TruexBASLZZ19}} & \textbf{--} & \textbf{--} & $0.81$ & $5.48$ & $0.23$ & $1.53$ & $23.15$ & $4.70$ & $34.37$ \\
&& \textbf{HybridAlpha \cite{XuBZAL19}} & \textbf{--} & \textbf{--} & $358.74$ & $550.51$ & $0.61$ & $1.33$ & $132.65$ & $303.35$ & $1350.18$ \\
&& \textbf{PILE \cite{TangSLZXQ23}} & $0.01$ & $0.02$ & $896.47$ & $2227.14$ & $5.56$ & $1.49$ & $502.17$ & \textbf{--} & $3635.48$ \\
&& \textbf{Our scheme} & $0.01$ & $0.01$ & $0.05$ & \textbf{--} & \textbf{--} & $1.51$ & \textbf{--} & $0.68$ & $2.19$ \\



\cline{2-12} 
&\multirow{5}{*}{\textbf{K=100}}
& \textbf{FedAVG \cite{McMahanMRHA17}} & \textbf{--} & \textbf{--} & $0.07$ & \textbf{--} & \textbf{--} & $1.40$ & \textbf{--} & \textbf{--} & $1.47$ \\
&& \textbf{TP-SMC \cite{TruexBASLZZ19}} & \textbf{--} &\textbf{--}& $35.82$ & $5.75$ & $0.31$ & $1.46$ & $28.29$ & $5.02$ & $76.65$ \\
&& \textbf{HybridAlpha \cite{XuBZAL19}} & \textbf{--} &\textbf{--}& $3581.74$ & $1021.16$ & $40.38$ & $1.24$ & $137.60$ & $308.87$ & $5077.49$ \\
&& \textbf{PILE \cite{TangSLZXQ23}} & $0.01$ & $0.19$ & $6485.76$ & $3659.73$ & $120.50$ & $1.41$ & $539.25$ & \textbf{--} & $10802.59$ \\
&& \textbf{Our scheme} & $0.01$ & $0.17$ & $0.08$ & \textbf{--} & \textbf{--} & $1.41$ & \textbf{--} & $0.77$ & $2.44$ \\

\midrule 
\multirow{10}{*}{\textbf{CIFAR100}}
&\multirow{5}{*}{\textbf{K=5}}
& \textbf{FedAVG \cite{McMahanMRHA17}} 
& \textbf{--} & \textbf{--} & $0.25$ & \textbf{--} & \textbf{--} & $5.36$ & \textbf{--} & \textbf{--} & $5.61$  \\
&& \textbf{TP-SMC \cite{TruexBASLZZ19}}
& \textbf{--} & \textbf{--} & $3.35$ & $11.82$ & $0.32$ & $5.82$ & $50.96$ & $27.46$ & $68.74$  \\
&& \textbf{HybridAlpha \cite{XuBZAL19}} 
& \textbf{--} & \textbf{--} & $1906.74$ & $1730.26$ & $0.69$ & $5.40$ & $564.85$ & $808.13$ & $5100.72$  \\
&& \textbf{PILE \cite{TangSLZXQ23}} 
& $0.02$ & $0.01$ & $2762.33$ & $4077.67$ & $5.65$ & $5.73$ & $1455.40$ & \textbf{--} & $8306.81$  \\
&& \textbf{Our scheme} 
& $0.02$ & $0.01$ & $0.28$ & \textbf{--} & \textbf{--} & $5.77$ & \textbf{--} & $0.59$ & $6.67$  \\



\cline{2-12} 
&\multirow{5}{*}{\textbf{K=100}}
& \textbf{FedAVG \cite{McMahanMRHA17}} 
& \textbf{--} & \textbf{--} & $0.54$ & \textbf{--} & \textbf{--} & $5.38$ & \textbf{--} & \textbf{--} & $5.92$  \\
&& \textbf{TP-SMC \cite{TruexBASLZZ19}}
& \textbf{--} & \textbf{--} & $69.21$ & $12.17$ & $0.38$ & $5.33$ & $54.73$ & $21.48$ & $163.30$  \\
&& \textbf{HybridAlpha \cite{XuBZAL19}} 
& \textbf{--} & \textbf{--} & $13989.87$ & $1417.68$ & $112.93$ & $5.25$ & $576.01$ & $836.78$ & $16938.52$  \\
&& \textbf{PILE \cite{TangSLZXQ23}} 
& $0.03$ & $0.18$ & $18729.23$ & $4275.51$ & $356.89$ & $5.06$ & $1492.73$ & \textbf{--} & $24859.63$  \\
&& \textbf{Our scheme} 
& $0.02$ & $0.17$ & $0.53$ & \textbf{--} & \textbf{--} & $5.93$ & \textbf{--} & $0.52$ & $7.18$  \\

	\bottomrule
	\end{tabular}

\end{table*}

\subsection{Proof of Theorem \ref{theo:para}}\label{appendix_theorem6}
\begin{proof}
    Firstly, note that given $\{\bm{r}^{(l)}\circ \bm{s}^{(l)}\}_{l=1}^{L}$,  for arbitrary noises $\{\bm r^{(l)}\in \mathbb{R}_{>0}^{n_{l}}\}^{L}_{l=1}$, there exist $\{\bm{s}^{(l)}\in \mathbb{R}_{>0}^{n_{l}}\}_{l=1}^{L}$. Besides, given $\{\widehat W^{(2l-1)}\}^{L}_{l=1}$ and $\{\widehat W^{(2l)}=diag(1/(\bm{r}_1^{(l)}\bm{s}_1^{(l)}), 1/(\bm{r}_2^{(l)}\bm{s}_2^{(l)}), \cdots, 1/(\bm{r}_{n_{l}}^{(l)}\bm{s}_{n_{l}}^{(l)}))\}^{L-1}_{l=1}$,  then for arbitrary noises $\{\bm r^{(l)}\}^{L}_{l=1}$ , there exist  global model parameters $W=\{W^{(l)}\}^{2L-1}_{l=1}$ such that the corresponding perturbed global model parameters are $\widehat {W}=\{\widehat W^{(l)}\}^{2L-1}_{l=1}$. 
Precisely, we construct the model parameters $W=\{ W^{(l)}\}^{2L-1}_{l=1}$ as follows:  
\begin{equation*}
\left\{
\begin{aligned}
 &  W^{(2l-1)}=\frac{1}{R^{(2l-1)}}\circ\widehat W^{(2l-1)}, \textnormal{ for } l=1,2,\dots, L-1,\\
& W^{(2l)}=I^{(l)}, \textnormal{ for } l=1,2,\dots, L-1,\\
 & W^{(2L-1)}=\frac{1}{R^{(2L-1)}}\circ(\widehat W^{(2L-1)}-R^{(a)}),
\end{aligned}\right.
\end{equation*}
and the noises $\{R^{(a)}, \{R^{(l)}\}_{l=1}^{2L-1}\}$ are given as
    \begin{equation*}\label{eq:newconstraint:Rrevised}
        \left\{
         \begin{aligned}
            &   R^{(1)}_{ij}=\bm r^{(1)}_i\\
            &  R^{(2l-1)}_{ij}=\bm r_i^{(l)}\bm{s}^{(l-1)}_{j}, \textnormal{ if $1 < l \leq L-1$} \\
            &  R^{(2l)}_{ij}=\left\{
                \begin{aligned}
                    &\frac{1}{\bm{s}^{(l)}_{i}\bm r_i^{(l)}}, \textnormal{ if $i=j$}\\
                    & 0, ~~~~~~~\textnormal{ if $i\neq j$} \\
                \end{aligned}\right.,\textnormal{if $1 \leq l \leq L-1$}\\
            &   R^{(2L-1)}_{ij}=\bm{s}_{j}^{(L-1)}\\
            & R^{(a)}_{ij}= \bm \gamma_i \cdot \bm r^{(a)}_{i}, \\
        \end{aligned}\right. 
    \end{equation*}
Then, for $1\leq l\leq L-1$, we have $R^{(2l-1)} \circ W^{(2l-1)} = R^{(2l-1)} \circ\frac{1}{R^{(2l-1)}}\circ\widehat W^{(2l-1)}=\widehat W^{(2l-1)}$,~
$R^{(2l)} \circ W^{(2l)} = R^{(2l)} \circ I^{(l)}=\widehat W^{(2l)}$,
and $R^{(2L-1)}  \circ W^{(2L-1)} + R^{(a)} = R^{(2L-1)}\circ\frac{1}{R^{(2L-1)}}\circ\widehat W^{(2L-1)} - R^{(2L-1)}\circ\frac{1}{R^{(2L-1)}}\circ R^{(a)}+R^{(a)}=\widehat W^{(2L-1)}.$
 Due to the arbitrariness of  $\{\bm r^{(l)}\}_{l=1}^{L}$ and $\bm r^{(a)}$, the set \big\{$\{ W^{(2l-1)}\}^{L}_{l=1}$: the corresponding  perturbed global model parameters are the given $\widehat {W}=\{\widehat W^{(l)}\}^{2L-1}_{l=1}$ \big\} is a positive measure set. However, \{$\{ W^{(2l-1)}\}^{L}_{l=1}$: $W=\{ W^{(2l-1)}\}^{L}_{l=1}$ are exactly the true global model parameters\} is a zero measure set. Therefore, the possibility that the clients obtain the true global model parameters equals 0.
\end{proof}

\section{Other Experimental Results}
\label{appendix_performance_detaile}
\subsection{Computational Efficiency Details}
\label{appendix_Computational_efficiency_details}
To explain the complexity of the proposed scheme, especially in terms of setup and maintenance of noise mechanisms, we add detailed comparisons in Tab. \ref{tab:time_cost_detail}. It shows the time taken for specific operations on both the server and client sides, including model perturbation (MP), recovery (Recover), aggregation (Agg), gradient decryption (Grad\_Dec), key generation (Key\_Gen), training (Train), gradient encryption (Grad\_Enc), and differential privacy (DP). The ``Total Time'' column summarizes the time required for one epoch of training. 

From the table, we can see that the computational costs of our $\mathbf{MP\_Server}$ ($\sim0.01$ for CIFAR10 and $\sim0.02$ for CIFAR100) and $\mathbf{DDP\_Client}$ ($\sim0.7$ for CIFAR10 and CIFAR100) are essentially negligible. Furthermore, our scheme does not introduce any additional overhead for model training or gradient aggregation, as experimental results show that the time for ``Agg'' and ``train'' in our method is nearly identical to that of FedAvg. However, the costs of introducing privacy protection in SOTAs are very expensive.

\subsection{Reconstruct Attack Image Results}
\label{appendix_Reconstruct}
We evaluate the robustness of our proposed scheme against reconstruction attacks, where adversaries attempt to recover the original training data from global model. Tab. \ref{tab:reconstructionResult} provides a visual comparison of reconstructed images across various methods. The ``Target'' represents the original ground-truth images, serving as a baseline for comparison.
It demonstrates that the reconstruction results of our scheme are essentially unrecognizable and devoid of meaningful information, while SOTAs exhibit varying degrees of reconstruction fidelity, revealing significant information about the original dataset. As a result, our scheme achieves a superior level of privacy preservation.

\begin{table}[htbp]
\centering
\caption{The Visualization of Reconstructed Results}
\label{tab:reconstructionResult}
\begin{tabular}{m{2.5cm}<{\centering}|m{1.1cm}<{\centering}m{1.1cm}<{\centering}m{1.1cm}<{\centering}}
  	\toprule
\textbf{Method} & \multicolumn{3}{c}{\textbf{Reconstructed Images}} \\
	\midrule
\textbf{Target} & \includegraphics[width=1.0cm]{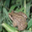} & \includegraphics[width=1.0cm]{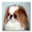} & \includegraphics[width=1.0cm]{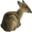} \\
\textbf{FedAVG \cite{McMahanMRHA17}} & \includegraphics[width=1.0cm]{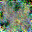} & \includegraphics[width=1.0cm]{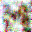} & \includegraphics[width=1.0cm]{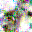} \\
\textbf{PILE \cite{TangSLZXQ23}}      & \includegraphics[width=1.0cm]{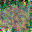} & \includegraphics[width=1.0cm]{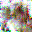} & \includegraphics[width=1.0cm]{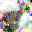} \\
\textbf{TP-SMC \cite{TruexBASLZZ19}} & \includegraphics[width=1.0cm]{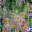} & \includegraphics[width=1.0cm]{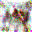} & \includegraphics[width=1.0cm]{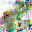} \\
\textbf{HybridAlpha \cite{XuBZAL19}} & \includegraphics[width=1.0cm]{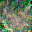} & \includegraphics[width=1.0cm]{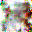} & \includegraphics[width=1.0cm]{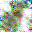} \\
\textbf{MP-CDP} & \includegraphics[width=1.0cm]{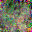} & \includegraphics[width=1.0cm]{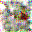} & \includegraphics[width=1.0cm]{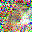} \\
\textbf{Our scheme} & \includegraphics[width=1.0cm]{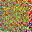} & \includegraphics[width=1.0cm]{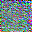} & \includegraphics[width=1.0cm]{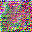} \\

     \bottomrule
\end{tabular}
\end{table}

\end{document}